\pdfoutput=1
\documentclass{article}

\usepackage[final]{corl_2026}

\usepackage{amsmath,amssymb,amsthm}
\usepackage{bm}
\usepackage{booktabs}
\usepackage{graphicx}
\usepackage{pdflscape}
\usepackage{wrapfig}


\definecolor{bcred}{rgb}{0.84,0.00,0.00}
\definecolor{qmaxorange}{rgb}{0.94,0.51,0.00}
\definecolor{awlsgreen}{rgb}{0.11,0.62,0.19}

\newtheorem{proposition}{Proposition}
\newtheorem{remark}{Remark}

\title{{AWLS}: Advantage-Weighted Latent Steering\\for One-Step Generative Robotic Manipulation\\Policies}

\author{
  Jiahong Tang$^{1,2}$ \And Jiaju He$^{1,2}$ \And Zhiyuan Yang$^{1,2}$ \And
  Ruyang Han$^{1,2}$ \And Haiyue Zhu$^{2\ddagger}$
  \AND
  \mdseries $^{1}$Nanyang Technological University, Singapore \\
  $^{2}$Singapore Institute of Manufacturing Technology (SIMTech), \\
  Agency for Science, Technology and Research (A*STAR), Singapore \\[2pt]
  $^{\ddagger}$Corresponding author
}

\hypersetup{
  pdftitle={AWLS: Advantage-Weighted Latent Steering for One-Step Generative Robotic Manipulation Policies},
  pdfauthor={Jiahong Tang, Jiaju He, Zhiyuan Yang, Ruyang Han, Haiyue Zhu},
  pdfkeywords={Robot Manipulation, Generative Policies, Latent Steering}
}

\begin{document}

\label{pg:main-first}%
\raggedbottom
\maketitle

\begin{abstract}
Human-collected real-world robotic manipulation demonstrations often contain
mixed-quality trajectories, including hesitant, indirect, or inefficient
action segments.
While diffusion- and flow-based policies can model such multimodal data,
standard behavior cloning (BC) assigns equal likelihood to all demonstrated
actions and may therefore reproduce suboptimal modes.
Offline RL value estimates offer a natural signal for biasing policies
toward higher-value behavioral modes, but directly fine-tuning the
generative policy or using critic gradients for policy steering can amplify
critic errors and induce prior drift.
We propose Advantage-Weighted Latent Steering (AWLS), an offline
post-training method for improving frozen one-step generative manipulation
policies without updating their decoders.
AWLS learns a residual noise shaper, an observation-conditioned residual
correction to the policy's latent noise, and trains it with an
advantage-weighted action reconstruction objective.
The advantages are detached, scale-normalized, and used only as supervised
weights on demonstrated actions, avoiding critic-gradient policy updates
while preserving fast single-step inference.
We further introduce a rollout-free decoder-sensitivity criterion for
selecting pretrained checkpoints that remain responsive to latent-noise
perturbations.
Experiments on simulated and real-world robotic manipulation tasks with
clean and mixed-quality demonstrations show that AWLS improves over the
frozen base policy while limiting drift from the pretrained behavior prior.
Project page: \url{https://awls-corl2026.github.io/}.
\end{abstract}


\section{Introduction}

Generative policies based on diffusion and flow
matching~\cite{chi2023diffusion, ze2024dp3, zhao2026fdvla} have become a
standard approach to imitation learning for robotic manipulation.
They are typically trained by behavior cloning (BC), which fits the policy to
human-collected demonstrations and therefore inherits their quality.
Yet even successful demonstrations are rarely uniformly optimal.
Some trajectories are smooth and efficient, while others are hesitant,
indirect, or unnecessarily long~\cite{brown2020better, mandlekar2021what}.
Standard BC does not distinguish between these quality levels, so it
reproduces hesitant behavior as faithfully as efficient behavior
(Figure~\ref{fig:teaser}, left).
Can a policy trained on such data instead be steered toward the better
behavior in those demonstrations?

\begin{figure}[!t]
\centering
\includegraphics[width=\textwidth]{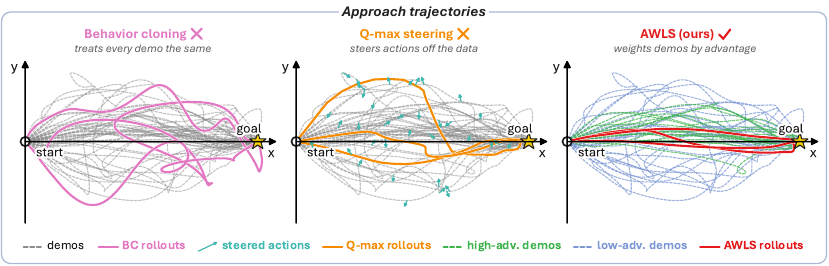}
\caption{Trajectory segments from human-collected robomimic can
  demonstrations~\cite{mandlekar2021what}, with policy rollouts.
  Thin dashed lines are demonstration segments from a shared start to the
  goal, and thick curves are illustrative policy rollouts.
  Behavior cloning treats every demonstration the same, so its rollouts
  reproduce the hesitant detours as faithfully as the efficient paths.
  Q-maximization steering pushes actions toward high-$Q$ estimates (arrows),
  so its rollouts leave the demonstrated region.
  AWLS weights demonstrations by advantage (green high, blue low) and reshapes
  the latent noise of the frozen decoder toward the high-advantage demonstrated
  mode, so its rollouts follow the efficient paths.}
\label{fig:teaser}
\vspace{-6pt}
\end{figure}

Offline reinforcement learning (RL) offers an answer: value estimates learned
from the same demonstrations can identify higher-advantage demonstrated
actions in their respective contexts.
The question then becomes how to inject this signal into a pretrained
generative policy.
Guiding the sampling process with the value
signal~\cite{janner2022diffuser, ajay2023decision, nakamoto2024vgps} keeps the
critic in the loop at inference time, and fine-tuning the decoder with value
gradients can amplify offline critic errors while eroding the pretrained
prior~\cite{ren2025dppo, sobolmark2024parl}.
This makes frozen-decoder latent steering an appealing interface: the
pretrained decoder is kept fixed, and improvement is attempted by reshaping
the latent noise.
Yet this interface exposes two hidden failure modes.
First, critic-gradient (Q-maximization) steering in latent
space~\cite{dsrl2025, im2026lps, patil2026goldenticket, lu2026unifiednoise}
can still chase extrapolated value artifacts rather than demonstrated
high-value behavior (Figure~\ref{fig:teaser}, middle).
Second, latent steering is only effective if the frozen decoder remains
sensitive to its latent noise, a property that behavior cloning may weaken
during training.

We propose Advantage-Weighted Latent Steering (AWLS) to mitigate both failure
modes simultaneously.
AWLS post-trains the policy with the pretrained decoder frozen, learning only
a residual noise shaper that applies observation-conditioned corrections to
the latent-noise input.
These corrections are initialized near zero to bias post-training toward the
original latent prior.
Critically, the offline critic never provides a gradient to the policy:
advantages estimated via implicit Q-learning (IQL)~\cite{kostrikov2022iql}
are detached and enter
training only as scalar weights on demonstrated
actions~\cite{peng2019awr} (Figure~\ref{fig:teaser}, right)---critic errors
may change which demonstrated actions are emphasized without directly
providing gradients toward potentially unsupported actions.
Two additional mechanisms help make this interface practical across tasks:
scale-normalized advantage weighting bounds weight concentration
independently of reward magnitude, and a rollout-free decoder-sensitivity
diagnostic selects frozen checkpoints that remain responsive to latent-noise
perturbations.

Our contributions are:
\textbf{(i)}~we identify two reliability issues in frozen latent steering
  for one-step generative manipulation policies, critic-gradient failure
  modes and decoder-sensitivity degradation, and address them with detached
  advantage weights and decoder-sensitivity checkpointing;
\textbf{(ii)}~we propose AWLS, a residual noise shaper trained by
  advantage-weighted reconstruction with scale-normalized weights while
  keeping the decoder frozen;
\textbf{(iii)}~we evaluate AWLS on simulation benchmarks, controlled
  mixed-quality demonstrations, and real-world robot tasks.

\section{Related Work}
\label{sec:related}

\paragraph{Generative policy learning.}
Diffusion and flow policies learn expressive action distributions for
manipulation~\cite{wang2025tocdp, wang2026dream}, using RGB and depth
perception~\cite{chi2023diffusion, ze2024dp3, yang2025hfdteacher} to condition
action generation.
Few- and one-step variants reduce sampling cost for fast closed-loop
control~\cite{prasad2024consistencypolicy, wang2025onedp}.
Building on these policy classes, offline RL uses learned value functions to
guide policy training and action
selection~\cite{wang2023diffusionql, hansenestruch2023idql, fql2025}.
Pretrained policies can also be adapted through offline-to-online RL,
policy-gradient updates, or generative-model
optimization~\cite{sobolmark2024parl, zhang2025gmpo, ren2025dppo,
mcallister2026fpo, domingo2025adjoint, zheng2026diffusionnft}.
In offline settings, value-gradient updates can exploit critic extrapolation
errors and move the policy away from its behavior-cloning
prior~\cite{kumar2020cql}.
Behavior-cloning regularization can reduce this drift, but the policy update
still depends on critic gradients.

\paragraph{Frozen-decoder post-training.}
A complementary line keeps the pretrained decoder fixed.
Inference-time guidance steers multi-step denoising with value functions,
return conditioning, or dynamics guidance~\cite{janner2022diffuser,
ajay2023decision, nakamoto2024vgps, du2025dynaguide}, but needs an iterative
sampler that one-step generators do not expose.
Noise-space steering instead adapts the latent noise of a frozen decoder
through critic gradients, rollout search, or external
feedback~\cite{dsrl2025, im2026lps, patil2026goldenticket,
lu2026unifiednoise}; critic-gradient variants still depend on the learned
value landscape, and search- and feedback-based variants need additional
evaluation signals.
Latent-action offline RL~\cite{zhou2021plas, chen2022lapo} trains a
VAE-style latent policy with value gradients or advantage weights, and
learned noise priors~\cite{ki2025pg} guide an iterative diffusion sampler.
AWLS differs from all of these in how it uses the latent interface.
Following advantage-weighted imitation~\cite{peng2019awr, nair2020awac,
wang2020crr}, offline advantages
estimated with IQL~\cite{kostrikov2022iql} serve only as detached weights for
supervised reconstruction of demonstrated actions.
AWLS therefore trains no new generative model, keeps single-step inference,
and needs neither value gradients nor critic-defined targets.

\section{Method}
\label{sec:method}

\begin{wrapfigure}{r}{0.56\textwidth}
  \vspace{-0.25\intextsep}% slight lift only; the cropped PDF has no top margin
  \centerline{\includegraphics[width=\linewidth]{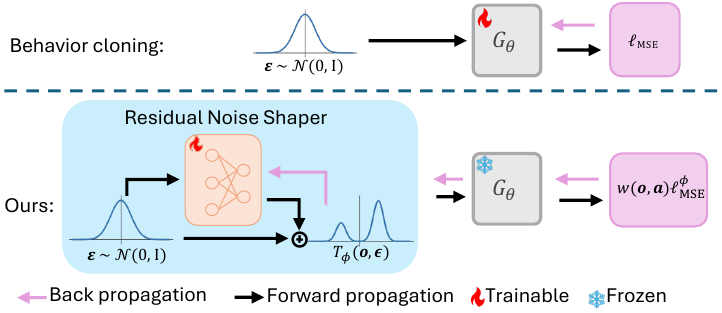}}
  \caption{\textbf{Overview of AWLS.}
    \emph{Top:} the one-step decoder $G_\theta$ is pretrained by behavior
    cloning. \emph{Bottom:} with $G_\theta$ frozen, $T_\phi$ reshapes the
    latent noise via advantage-weighted reconstruction using detached,
    normalized offline advantages.}
  \label{fig:method_overview}
\end{wrapfigure}
AWLS is a decoder-preserving post-training operator for one-step generative
robot policies.
It combines three components: the frozen decoder supplies the action prior,
an offline critic supplies only detached preferences over demonstrated
actions, and a residual noise shaper $T_\phi$, the only trainable module,
acts solely through the decoder's latent noise.
Figure~\ref{fig:method_overview} contrasts the two stages: the top row shows
standard behavior cloning, and the bottom row shows AWLS post-training through
a frozen decoder.
This section first states the problem through the one-step interface, then
derives the advantage-weighted latent objective, the scale-normalized weighting
rule, and the checkpoint criterion needed for the interface to remain
effective.

\subsection{Problem Statement: Frozen One-Step Decoders}
\label{sec:method_problem}

We consider offline post-training for a pretrained one-step generative
policy~\cite{lipman2023flow, geng2025meanflow, sheng2026mp1, zhan2026mvp}
that maps the current observation $\bm{o}$ and Gaussian latent noise
$\bm{z}$ to an action chunk $\bm{a}$:
\begin{equation}
  \label{eq:onestep_decoder}
  \bm{a}=G_\theta(\bm{o},\bm{z}),
  \qquad
  \bm{z}\sim\mathcal{N}(\bm{0},\bm{I}),
\end{equation}
where \emph{one-step} means that one action chunk is produced by a single
evaluation of $G_\theta$.
Multi-step diffusion or flow-matching samplers are also differentiable, but
optimizing them through an action-space objective requires unrolling the
sampling chain; Eq.~\eqref{eq:onestep_decoder} instead exposes a direct,
unroll-free gradient path from the action-space loss back to the latent
noise.
The decoder $G_\theta$ is \emph{frozen} throughout post-training.

We are given an offline dataset $\mathcal{D}$ of mixed-quality
demonstrations, each providing transitions $(\bm{o},\bm{a},r,\bm{o}')$; the
critic below uses full transitions, whereas the policy loss uses only the
observation--action pairs $(\bm{o},\bm{a}_{\mathrm{demo}})$.
The reward is the same binary task-success reward used for evaluation: it is
emitted when the task first reaches its success condition, and the episode
terminates on success, so with discount $\gamma<1$ earlier success yields a
larger return.
These rewards are used only to fit the offline critic of
Section~\ref{sec:method_advantage}.
AWLS learns only a residual noise shaper $T_\phi$ that maps base noise
$\bm{\epsilon}\sim\mathcal{N}(\bm{0},\bm{I})$ to the decoder input
$\bm{z}=T_\phi(\bm{o},\bm{\epsilon})$, so that
$\bm{a}=G_\theta(\bm{o},T_\phi(\bm{o},\bm{\epsilon}))$ preferentially
reconstructs higher-quality demonstrated actions.
For a demonstrated pair $(\bm{o},\bm{a}_{\mathrm{demo}})$ we define the
noise-averaged mean squared error (MSE)
\begin{equation}
  \label{eq:awls_reconstruction}
  \ell^{\phi}_{\mathrm{MSE}}(\bm{o},\bm{a}_{\mathrm{demo}})
  = \mathbb{E}_{\bm{\epsilon}\sim\mathcal{N}(\bm{0},\bm{I})}\!\left[
      \big\|G_\theta\big(\bm{o},T_\phi(\bm{o},\bm{\epsilon})\big)
        -\bm{a}_{\mathrm{demo}}\big\|_2^2
    \right],
\end{equation}
which depends on $\phi$ only through $T_\phi$.
The identity shaper $T_\phi=\mathrm{id}$ recovers the pretrained BC policy;
Eq.~\eqref{eq:awls_reconstruction} defines the action reconstruction loss
used during post-training.

\subsection{Advantage-Weighted Latent-Steering Objective}
\label{sec:method_arch}
\label{sec:method_advantage}

Motivated by advantage-weighted regression (AWR)~\cite{peng2019awr}, AWLS
weights the reconstruction loss by exponentiated advantages estimated offline
via IQL~\cite{kostrikov2022iql} (Appendix~\ref{app:critic}).
A pessimistic critic ensemble gives
$Q_{\mathrm{pess}}(\bm{o},\bm{a})=\min_k Q_{\psi_k}(\bm{o},\bm{a})$, an
expectile-trained value function $V_\nu$ provides the baseline, and the
advantage is the detached difference
\begin{equation}
  \label{eq:awls_advantage}
  A(\bm{o},\bm{a})
  = \operatorname{sg}\!\big[Q_{\mathrm{pess}}(\bm{o},\bm{a})-V_\nu(\bm{o})\big],
\end{equation}
where $\operatorname{sg}[\cdot]$ denotes stop-gradient.
The raw AWR weight has the form
$w_{\mathrm{raw}}(\bm{o},\bm{a}) = \exp(\beta A(\bm{o},\bm{a}))$.
Likelihood-based AWR fits the reweighted behavior distribution
$\pi^*(\bm{a}\mid\bm{o})\propto
 \pi_{\mathcal{D}}(\bm{a}\mid\bm{o})\exp(\beta A(\bm{o},\bm{a}))$,
which assigns greater probability to higher-advantage demonstrated actions.
AWLS adopts this weighting principle, but the marginal log-density of its
composed policy is not directly tractable: even for dimension-matched
invertible maps, evaluation requires an inverse mapping and the Jacobian
log-determinants of both $G_\theta$ and
$T_\phi$~\cite{grathwohl2019ffjord}.
We therefore use the action reconstruction loss of
Eq.~\eqref{eq:awls_reconstruction} as an AWR-inspired supervised surrogate:%
\begin{equation}
  \label{eq:awr_loss}
  \mathcal{L}_{\text{AWLS}}(\phi)
    = \mathbb{E}_{(\bm{o},\bm{a}_{\text{demo}})\sim\mathcal{D}}\!\Big[
        w(\bm{o}, \bm{a}_{\text{demo}})
        \;\ell_{\mathrm{MSE}}^\phi(\bm{o},\bm{a}_{\text{demo}})
      \Big],
\end{equation}
where the weight
$w(\bm{o},\bm{a})=\exp(\beta\,\tilde{A}(\bm{o},\bm{a}))$ uses a
scale-normalized advantage $\tilde{A}\in[0,1]$ defined by
Eq.~\eqref{eq:awls_norm} in Section~\ref{sec:method_weighting}.
The critic affects training only through the detached weight
$w(\bm{o},\bm{a}_{\text{demo}})$: advantages, batch statistics, and weights
carry no gradient, so AWLS neither backpropagates through $Q_{\mathrm{pess}}$
nor uses an action-direction gradient $\nabla_{\bm{a}}Q_{\mathrm{pess}}$,
and the only gradient flows from the MSE loss through the frozen decoder into
$T_\phi$.

This objective anchors the training targets to demonstrated actions, but
does not establish equivalence to likelihood-based AWR or guarantee that
generated actions remain within the demonstration support.
Because independent noise samples are regressed toward demonstrated targets,
reconstruction can also reduce conditional action diversity
(Appendix~\ref{app:surrogate}).
To bias adaptation toward the pretrained latent input, AWLS uses the residual
form
\begin{equation}
  \label{eq:residual_shaper}
  T_\phi(\bm{o},\bm{\epsilon})
    = \bm{\epsilon} + \Delta_\phi(\bm{o},\bm{\epsilon}),
\end{equation}
where $\Delta_\phi$ is a small MLP conditioned on the observation encoding
and $\bm{\epsilon}$ (Appendix~\ref{app:impl_shaper}), with its output
initialized near zero, so post-training starts near the identity shaper.
If $\|\partial\Delta_\phi/\partial\bm{\epsilon}\|_{\mathrm{op}}\leq L_v<1$
at a state--noise pair, the shaper Jacobian is nonsingular and its singular
values lie in $[1-L_v,1+L_v]$ (Prop.~\ref{prop:surrogate}).
This controls local geometry, rather than the MSE--likelihood discrepancy.
The condition is not enforced during training; measurements at the evaluated
checkpoints show modest corrections and nonsingular local Jacobians, while
the residual shaper retains more action diversity than the tested direct and
Gaussian alternatives (Appendices~\ref{app:surrogate}
and~\ref{app:impl_shaper}).

\subsection{Scale-Normalized Advantage Weighting}
\label{sec:method_weighting}

A practical fragility of the raw AWR weight
$w_{\mathrm{raw}}=\exp(\beta A)$ is that
the effective sharpness $\beta\,A$ depends on the magnitude of the advantage,
which varies widely across tasks: different horizon lengths, reward
densities, and return magnitudes produce different advantage ranges even
under the same reward convention.
A fixed $\beta$ therefore has no consistent meaning: on a task with small
advantages the weights become nearly uniform, while on a task with large
advantages they overflow \texttt{float32} or collapse the effective sample
size (ESS) of the batch toward~$1$, concentrating all weight on a single
sample.
To decouple the inverse temperature $\beta$ from this task-specific advantage
scale, we normalize the advantage to the unit
interval before exponentiation.  Given a minibatch $\mathcal{B}$, define
\begin{equation}
  \label{eq:awls_norm}
  \tilde{A}(\bm{o}, \bm{a})
    = \frac{A(\bm{o},\bm{a}) - A_{\min}^{\mathcal{B}}}
           {A_{\max}^{\mathcal{B}} - A_{\min}^{\mathcal{B}} + \varepsilon},
  \qquad
  \tilde{A} \in [0, 1],
\end{equation}
where $A_{\min}^{\mathcal{B}}$ and $A_{\max}^{\mathcal{B}}$ are the batch
extrema and $\varepsilon > 0$ is a small constant for numerical
stability.
Note that both $\tilde{A}$ and $w$ depend on the current minibatch; we
suppress this dependence for brevity.
Substituting this normalized advantage into the weight $w$ of
Eq.~\eqref{eq:awr_loss} bounds the weight: since $\tilde{A} \in [0,1]$,
$w \in [1, e^\beta]$ regardless of the reward scale, advantage range, or
critic magnitude.
In contrast, the raw AWR weight $w_{\mathrm{raw}}$ produces a
within-batch weight ratio of
$e^{\beta(A_{\max}^{\mathcal{B}}-A_{\min}^{\mathcal{B}})}$, which grows
exponentially with the advantage range.

Equivalently, per-batch normalization applies an adaptive inverse temperature
$\beta_t = \beta / (\Delta_t + \varepsilon)$ to min-shifted advantages, where
$\Delta_t = A_{\max}^{\mathcal{B}_t} - A_{\min}^{\mathcal{B}_t}$.
We use $\beta=3$ for all tasks (Section~\ref{sec:ablations}).

\begin{proposition}[Bounded weight concentration]
  \label{prop:ess}
  For any minibatch of size $n$, the AWLS weights satisfy
  $w_{\max}/w_{\min} \leq R$ with $R = e^\beta$, and the effective sample
  size is lower-bounded by
  $\mathrm{ESS}
    = (\sum_i w_i)^2 / \sum_i w_i^2
    \geq 4nR/(1+R)^2$.
  For $\beta=3$ and $n=256$, this gives $\mathrm{ESS} \geq 46.3$.
  The bound holds at every training step regardless of the critic state.
\end{proposition}

\noindent
The bound controls worst-case weight concentration only; it does not
guarantee that the critic ranking is accurate or informative.
The proof is in Appendix~\ref{app:ess}.

\subsection{Frozen Base Checkpoint Selection}
\label{sec:method_checkpoint}

Latent steering can change the decoded action only through the local
noise--action Jacobian
$J_\theta(\bm{o},\bm{z})=\partial G_\theta(\bm{o},\bm{z})/\partial\bm{z}$.
For a sampled $\bm{\epsilon}$ with $\bm{z}=T_\phi(\bm{o},\bm{\epsilon})$ and
$\hat{\bm{a}}=G_\theta(\bm{o},\bm{z})$, the gradient of the squared error
$\ell=\|\hat{\bm{a}}-\bm{a}_{\mathrm{demo}}\|_2^2$ that reaches the shaper is
the action-space gradient pulled back through the decoder,
$\nabla_\phi \ell
 = (\partial T_\phi/\partial\phi)^{\top} J_\theta(\bm{o},\bm{z})^{\top}
   \nabla_{\hat{\bm{a}}}\ell$.
When $J_\theta$ has low gain or is concentrated in a few directions, this
gradient provides little leverage for a latent-only update.
We observe that such loss of latent sensitivity can occur late in BC training
even after task performance has plateaued (Appendix~\ref{app:ckpt}), so the
final BC checkpoint is not necessarily the best frozen base for latent
steering.
We therefore score the checkpoints $\{\theta_k\}$ saved during BC training
with two rollout-free diagnostics computed on held-out demonstrations.
\emph{Best-of-$M$ reconstruction error} $L_{\text{BoM}}$ proxies BC quality
by the best action-space match over $M$ latent-noise samples.
\emph{Directional sensitivity} $S_J$ is a finite-difference estimate of the
gain of $J_\theta$: the decoded action change under small random
latent-noise perturbations.
Among checkpoints whose BC quality is within a tolerance $\delta$ of the
final checkpoint ($\delta=0.2$ throughout), we select the one with the
highest decoder sensitivity:
\begin{equation}
  \label{eq:ckpt_select}
  \theta^*
    = \arg\max_{\theta_k\in\mathcal{C}_\delta} S_J(\theta_k),
  \qquad
  \mathcal{C}_\delta
    = \big\{\theta_k :
        L_{\text{BoM}}(\theta_k)
          \leq (1+\delta)\,L_{\text{BoM}}(\theta_{\text{final}})
      \big\}.
\end{equation}
Formal definitions and per-task diagnostics are provided in
Appendix~\ref{app:ckpt}.

\paragraph{Optimization and deployment.}
Each post-training step samples a demonstration minibatch, computes detached
IQL advantages and their normalized weights, draws one independent base-noise
sample per pair as a single-sample estimate of
Eq.~\eqref{eq:awls_reconstruction}, and takes a gradient step on
Eq.~\eqref{eq:awr_loss} with respect to $\phi$ only.
At deployment the critic and advantage computation are removed: AWLS samples
$\bm{\epsilon}$, evaluates $T_\phi(\bm{o},\bm{\epsilon})$, and performs one
decoder evaluation, so inference remains a single generative step without
iterative denoising or candidate selection.

\section{Experiments}
\label{sec:experiments}

\subsection{Simulation Experiments}
\label{sec:exp_benchmark}

\paragraph{Setup.}
We evaluate AWLS on Adroit~\cite{rajeswaran2018learning},
DexArt~\cite{bao2023dexart}, and RoboTwin~\cite{mu2025robotwin}, covering 12
simulation tasks with single-hand dexterity, articulated-object interaction,
and long-horizon bimanual manipulation.
Demonstrations are generated by simulation RL expert policies.
Although these experts achieve high success, their rollouts still vary in
trajectory length, smoothness, contact timing, and hesitation, so simulation
serves as a controlled testbed for the within-success quality variation that
makes real-world teleoperated data difficult.
We use 10 demonstrations per Adroit task, 100 per DexArt task, and 50 per
RoboTwin task.
All simulation experiments are run on a single NVIDIA RTX 3090 GPU.

\paragraph{Baselines.}
We compare with behavior-cloning policies
DP3~\cite{ze2024dp3}, FlowPolicy~\cite{zhang2025flowpolicy}, and
MeanFlow~\cite{geng2025meanflow}, as well as the Q-maximization post-training
methods FQL~\cite{fql2025} and DSRL~\cite{dsrl2025}, built on FlowPolicy and
DP3 backbones, respectively.
FP-AWLS applies our post-training procedure to a FlowPolicy-style one-step
decoder, while MF-AWLS applies it to a MeanFlow decoder.
All methods use the same offline demonstrations.

\paragraph{Post-training setup.}
We first train each one-step decoder by behavior cloning, for 3000 epochs on
Adroit and DexArt and 1000 epochs on RoboTwin (per-benchmark settings in
Appendix~\ref{app:impl_decoder}).
FP-AWLS and FQL start from the same pretrained FlowPolicy decoder, so their
comparison isolates the post-training procedure.
Both AWLS variants select their frozen base from the BC checkpoints with the
rollout-free $S_J$ criterion of Section~\ref{sec:method_checkpoint}
(Eq.~\eqref{eq:ckpt_select}); AWLS then freezes the selected decoder and
updates only the residual noise shaper~$T_\phi$.
The offline critic is trained on the same demonstrations.
Each method is trained across three random seeds, and every reported
simulation success rate is the mean over $20$ evaluation episodes per task
per seed.

\begin{table}[!t]
{\nolinenumbers
\begin{center}
\begingroup
\scriptsize
\setlength{\tabcolsep}{2.7pt}
\begin{tabular}{ccccccc}
\toprule
Paradigm & Method & Denoising Steps & Adroit (3 tasks) & DexArt (4 tasks) & RoboTwin (5 tasks) & Average \\
\midrule
\raisebox{-1.15\normalbaselineskip}[0pt][0pt]{\shortstack{Behavior\\Cloning}}
  & DP3 & 10 & 67.5\(\pm\)3.3 & 50.3\(\pm\)1.2 & 45.0\(\pm\)3.8 & 52.4\(\pm\)2.8 \\
  & FlowPolicy & 1 & 71.3\(\pm\)2.0 & 50.4\(\pm\)2.2 & 48.3\(\pm\)4.2 & 54.8\(\pm\)3.0 \\
  & MeanFlow & 1 & 72.1\(\pm\)4.2 & 54.7\(\pm\)2.2 & 53.2\(\pm\)4.6 & 58.4\(\pm\)3.7 \\
\midrule
\raisebox{-0.55\normalbaselineskip}[0pt][0pt]{Q-maximization}
  & DSRL & 10 & 72.5\(\pm\)2.5 & 56.4\(\pm\)3.7 & 55.8\(\pm\)4.8 & 60.2\(\pm\)3.9 \\
  & FQL & 1 & 73.6\(\pm\)2.6 & 57.4\(\pm\)2.4 & 56.1\(\pm\)4.9 & 60.9\(\pm\)3.5 \\
\midrule
\raisebox{-0.55\normalbaselineskip}[0pt][0pt]{Ours} & FP-AWLS
  & 1
  & 74.8\(\pm\)2.6
  & 58.5\(\pm\)2.1
  & 57.6\(\pm\)4.5
  & 62.2\(\pm\)3.2 \\
  & MF-AWLS
  & 1
  & \textbf{76.1\(\pm\)3.5}
  & \textbf{62.2\(\pm\)1.3}
  & \textbf{61.6\(\pm\)3.7}
  & \textbf{65.4\(\pm\)2.9} \\
\bottomrule
\end{tabular}
\endgroup
\par\vspace{2pt}
\refstepcounter{table}\label{tab:summary}
\parbox{0.96\textwidth}{\small
  \textbf{Table~\thetable:} Cross-task average success rate (\%) by benchmark.
  Each entry averages the per-task success rates of the benchmark (means over
  three random seeds); the $\pm$ value is the average of the per-task standard
  deviations (per-task results in Appendix~\ref{app:full_sim_results}).}
\end{center}
\linenumbers}

\end{table}

\paragraph{Results.}
Table~\ref{tab:summary} shows that both AWLS variants improve over their
frozen one-step base policies on every simulation benchmark family; per-task
breakdowns in Appendix~\ref{app:full_sim_results} confirm the trend across
all 12 tasks.
Under the matched FlowPolicy backbone, the FP-AWLS margin over FQL is modest
($62.2$ vs.\ $60.9$ on average); FQL's behavior-regularization coefficient is
tuned per task following its official implementation, whereas AWLS uses a
single fixed $\beta=3$ for all tasks.
Our claim is therefore competitive performance under a simpler,
shared-configuration frozen-decoder post-training protocol rather than a large
margin over tuned Q-maximization baselines.

\subsection{Mixed-Quality Demonstrations}
\label{sec:exp_mixed}

We next evaluate robustness to mixed-quality demonstrations on DexArt.
Starting from the same RL expert demonstrations as in
Section~\ref{sec:exp_benchmark}, we inject
temporally correlated Ornstein--Uhlenbeck (OU) action noise into expert
rollouts to construct a low-quality demonstration set.
The noise level is chosen so that the expert rollout success rate drops from
near 100\% to below 20\%, while still producing trajectories that remain close
to the original task structure (Appendix~\ref{app:mixed_quality}).
We then mix clean expert demonstrations and noisy expert-derived
demonstrations at different ratios and train each method on the resulting
offline dataset.
We use the same success-rate evaluation protocol as in
Section~\ref{sec:exp_benchmark}.

\begin{figure}[!t]
\centering
\begin{minipage}[b]{0.48\textwidth}
  \centering
  \includegraphics[width=\textwidth]{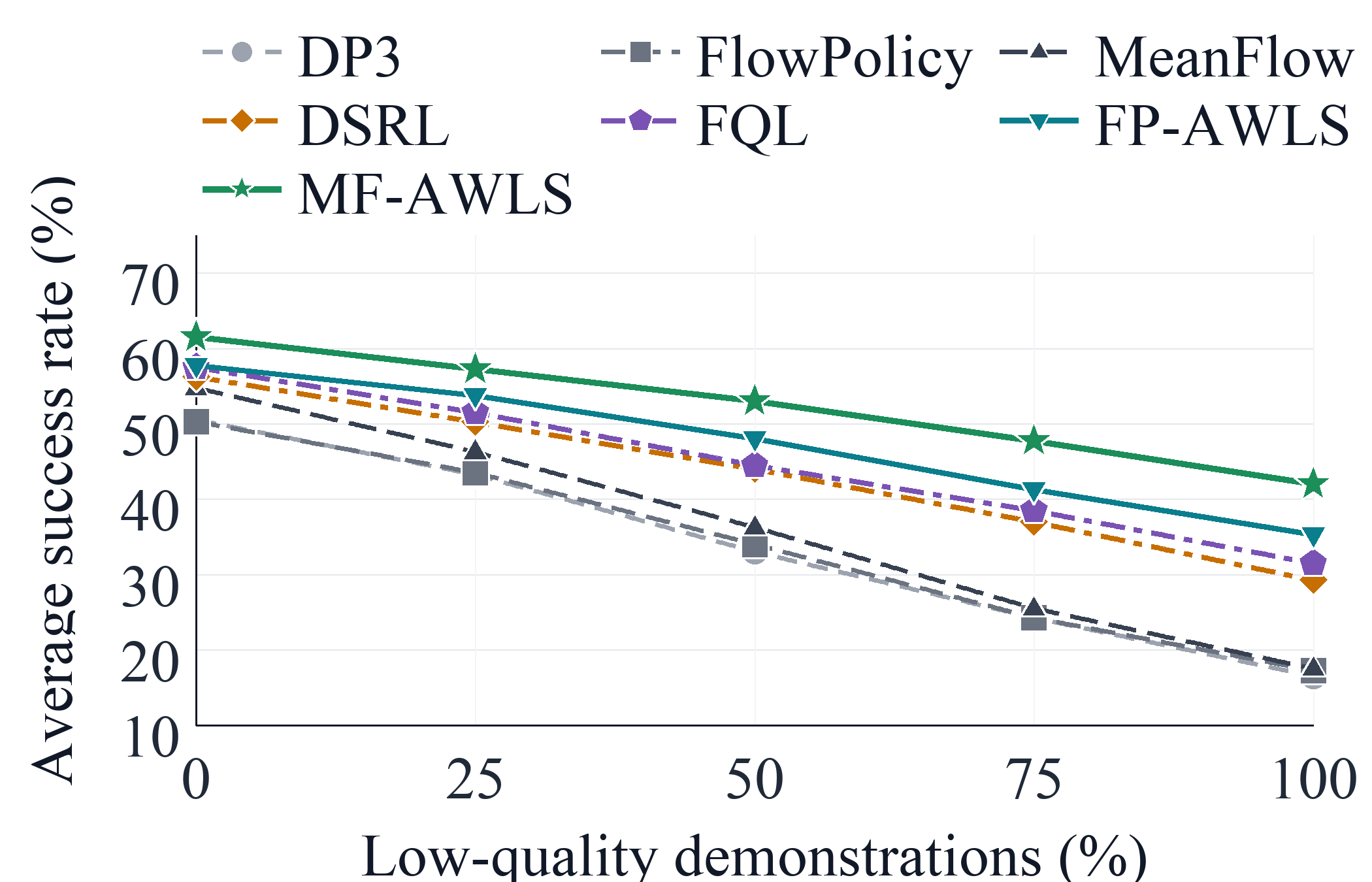}
\end{minipage}
\hfill
\begin{minipage}[b]{0.48\textwidth}
  \centering
  \includegraphics[width=\textwidth]{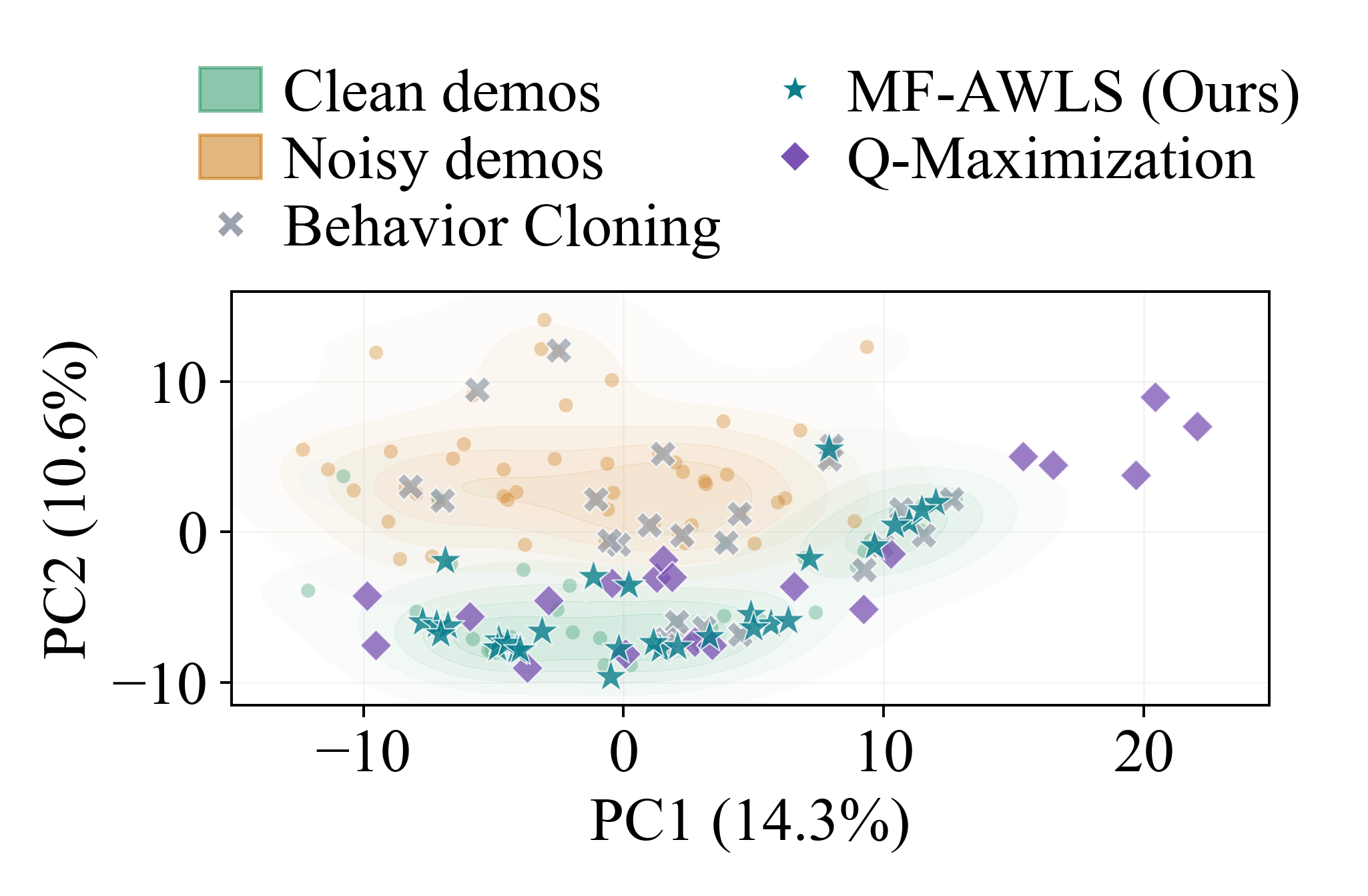}
\end{minipage}
\caption{\textbf{Left:} Average success rate on four DexArt tasks as the
  noisy-demonstration fraction increases.
  \textbf{Right:} PCA projection of flattened Bucket action trajectories
  with 50 clean and 50 noisy demonstrations.}
\label{fig:mixed_quality}
\end{figure}

As the noisy-demonstration fraction increases (Figure~\ref{fig:mixed_quality},
left), all methods degrade, but AWLS
maintains higher success than the behavior-cloning and Q-maximization
baselines.
The trajectory PCA (Figure~\ref{fig:mixed_quality}, right) further shows how the
methods use the mixed dataset.
Behavior cloning spreads generated trajectories across both clean and noisy
demonstration modes because it treats all demonstrations uniformly.
AWLS instead concentrates its trajectories near the clean-demonstration
cluster in this projection, consistent with a preference for higher-quality
demonstrated behavior in the mixed dataset.
Q-maximization also shifts many trajectories toward the clean mode, but a
nontrivial fraction lies outside both demonstration clusters in this projection,
consistent with critic-gradient steering inducing out-of-distribution drift.

\subsection{Real-World Experiments}
\label{sec:exp_realworld}

We further evaluate AWLS on real robots to test whether the same offline
post-training recipe transfers to policies trained from human teleoperated
demonstrations, the intended setting of AWLS.

\paragraph{Setup.}
For real-world experiments, the single-arm tasks use a UR5e manipulator,
whereas the dual-arm tasks use two UR3e manipulators.
The perception setup combines one third-view \mbox{Intel RealSense D455} camera
with wrist-mounted \mbox{Intel RealSense D405} cameras on each arm, providing a
shared scene view together with close-range hand-centric observations.
Policies consume these two RGB streams plus proprioception and output
end-effector pose-and-gripper commands in action chunks at $10$~Hz; the
offline critic uses the same binary task-success reward as the evaluation
criterion.
Hardware, observation encoding, and evaluation details are given in
Appendix~\ref{app:realworld}.

\paragraph{Tasks and demonstrations.}
Figure~\ref{fig:real_tasks} shows five real-world tasks: three single-arm
tasks and two bimanual tasks.
All methods are trained from the same offline teleoperated demonstrations:
each task uses $70$ demonstrations collected by SpaceMouse teleoperation, with
commands expressed in the robot base frame.
Even when demonstrations are successful, they vary in speed, contact timing,
pauses, corrections, and final alignment, creating the mixed-quality
conditions that advantage weighting is designed to address.

\begin{figure}[!t]
\centering
\begin{minipage}[b]{0.18\textwidth}
  \centering
  \includegraphics[width=\textwidth]{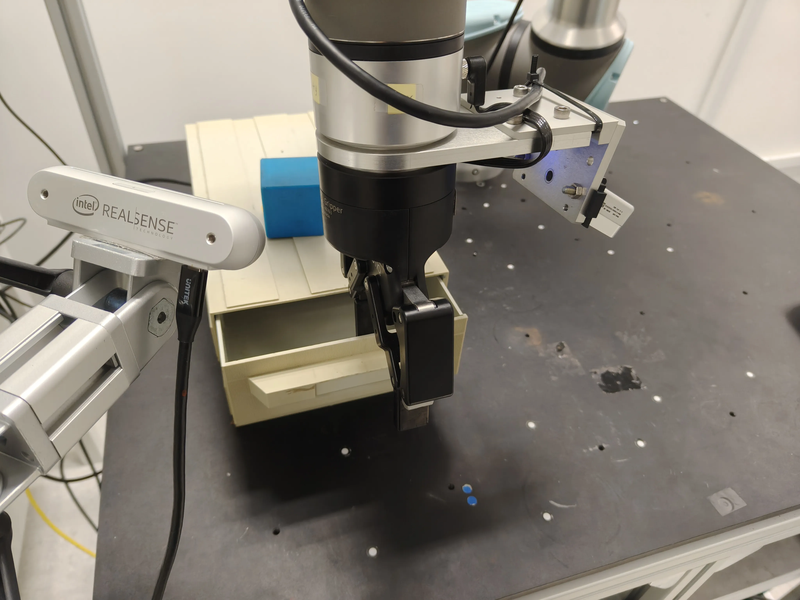}\\[2pt]
  \makebox[\linewidth][c]{\scriptsize Drawer-Cube}
\end{minipage}\hfill%
\begin{minipage}[b]{0.18\textwidth}
  \centering
  \includegraphics[width=\textwidth]{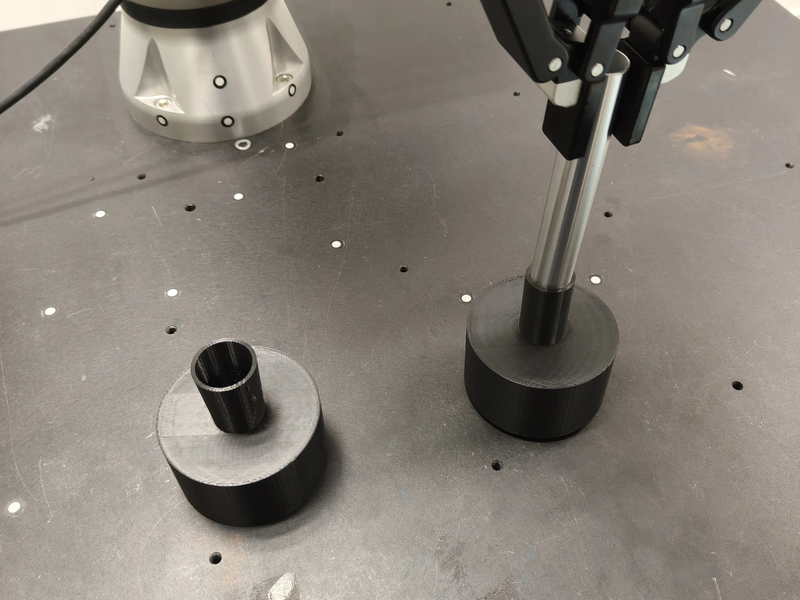}\\[2pt]
  \makebox[\linewidth][c]{\scriptsize Tube Insertion}
\end{minipage}\hfill%
\begin{minipage}[b]{0.18\textwidth}
  \centering
  \includegraphics[width=\textwidth]{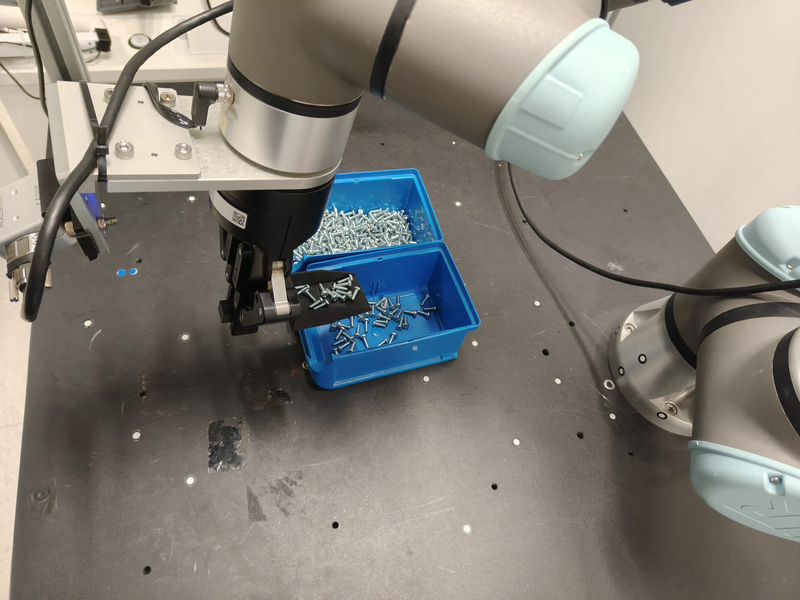}\\[2pt]
  \makebox[\linewidth][c]{\scriptsize Screw Scooping}
\end{minipage}\hfill%
\begin{minipage}[b]{0.18\textwidth}
  \centering
  \includegraphics[width=\textwidth]{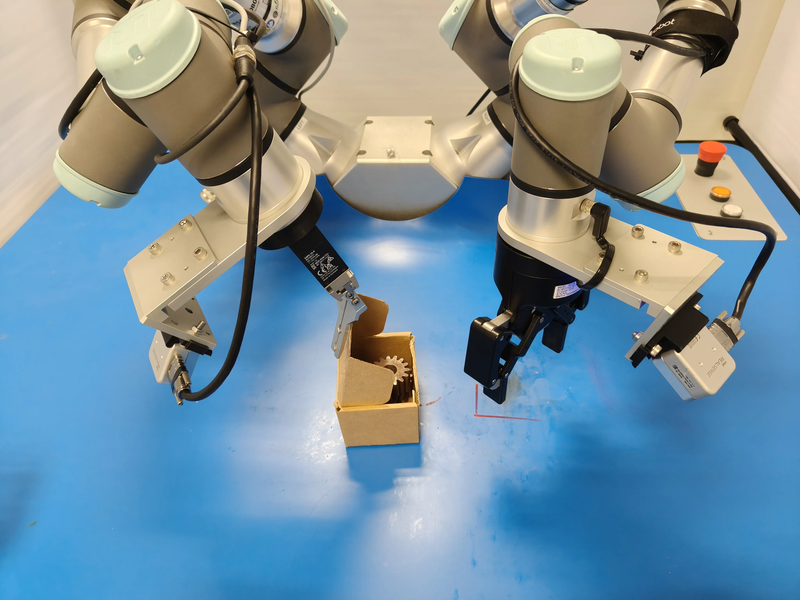}\\[2pt]
  \makebox[\linewidth][c]{\scriptsize Gear Packing}
\end{minipage}\hfill%
\begin{minipage}[b]{0.18\textwidth}
  \centering
  \includegraphics[width=\textwidth]{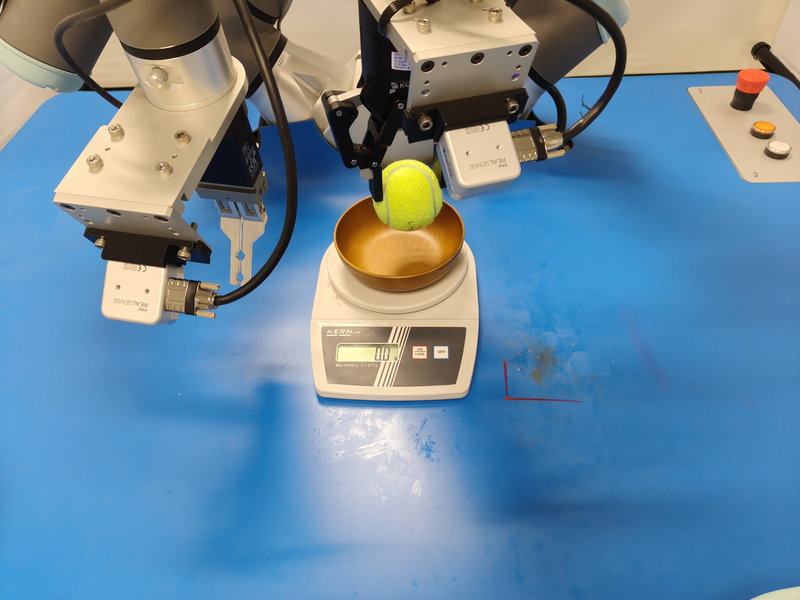}\\[2pt]
  \makebox[\linewidth][c]{\scriptsize Ball Weighing}
\end{minipage}
\caption{Real-world manipulation tasks. The first three are single-arm tasks:
  \textbf{Drawer-Cube} opens a drawer, places a cube, and closes the drawer;
  \textbf{Tube Insertion} removes a jar and inserts a tube;
  \textbf{Screw Scooping} scoops screws and pours them into a target container.
  The last two are dual-arm tasks:
  \textbf{Gear Packing} opens a box, places a gear, and closes the box;
  \textbf{Ball Weighing} presses the scale button with one arm while placing
  the ball on the scale with the other arm.}
\label{fig:real_tasks}
\end{figure}

\begin{table}[!t]
\centering
\setlength{\tabcolsep}{4.4pt}
\renewcommand{\arraystretch}{0.92}
\begin{tabular}{@{}lcccccc@{}}
\toprule
Method & Drawer-Cube & Tube Ins. & Screw Scoop. & Gear Pack. & Ball Weigh. & Total \\
\midrule
MeanFlow & 12/20 & 11/20 & 13/20 & 4/20 & 5/20 & 45/100 \\
FQL  & 13/20 & 14/20 & 15/20 & 5/20 & 4/20 & 51/100 \\
DSRL & 12/20 & 15/20 & 13/20 & 6/20 & 4/20 & 50/100 \\
Ours & \textbf{15/20} & \textbf{17/20} & \textbf{16/20} &
       \textbf{9/20} & \textbf{8/20} & \textbf{65/100} \\
\bottomrule
\end{tabular}
\caption{Real-world success counts over 20 hardware rollouts per task.
  MeanFlow is the frozen behavior-cloning base used by MF-AWLS.}
\label{tab:realworld}
\vspace{-1.4em}% tighten the gap to the text below (two stacked top floats)
\end{table}

\paragraph{Training and evaluation.}
We compare the MF-AWLS variant with its frozen MeanFlow behavior-cloning
base, FQL, and DSRL; no real-world rollouts are used during post-training.
Each policy is evaluated for 20 hardware rollouts per task from randomized but
feasible initial configurations.
A rollout is successful only if every task stage is completed without human
intervention.

\paragraph{Results.}
Table~\ref{tab:realworld} shows that AWLS improves over the frozen MeanFlow
base and obtains higher success counts than FQL and DSRL on all five real-world
tasks.
With $20$ rollouts per task the per-task intervals are wide; in a
task-stratified test with Holm correction, only the aggregate gain over the
frozen MeanFlow base ($65/100$ vs.\ $45/100$) is statistically significant,
while the margins over FQL and DSRL are suggestive
(Appendix~\ref{app:realworld}).
These results support the use of detached advantage weighting to improve
policies trained from successful human demonstrations while keeping the
pretrained decoder fixed.

\subsection{Ablation Studies}
\label{sec:ablations}

We ablate the design choices of AWLS on the four DexArt tasks with clean demonstrations, using a MeanFlow
backbone.
Table~\ref{tab:ablation} reports the four discrete comparisons (decoder
adaptation, noise-shaper parameterization, frozen-base checkpoint selection,
and advantage weighting); the continuous advantage-temperature sweep is
reported in Appendix~\ref{app:beta_sweep} (Figure~\ref{fig:beta}).

\begin{wraptable}{r}{0.5\textwidth}
  \vspace{-6pt}
  \centering
  \small
  \setlength{\tabcolsep}{6pt}
  \renewcommand{\arraystretch}{0.93}
\begin{tabular}{lc}
  \toprule
   & DexArt SR \\
  \midrule
  \textbf{MF-AWLS (ours)}            & $\mathbf{62.2}$ \\
  \midrule
  \multicolumn{2}{l}{\textit{(a) Decoder adaptation}}\\
  \quad BC (no post-training)        & $-6.7$ \\
  \quad Q-max steering               & $-6.7$ \\
  \quad Unfreeze decoder             & $-4.0$ \\
  \quad Unfreeze enc.\,+\,dec.       & $-9.7$ \\
  \addlinespace[3pt]
  \multicolumn{2}{l}{\textit{(b) Noise shaper}}\\
  \quad Direct $f_\phi$              & $-2.5$ \\
  \addlinespace[3pt]
  \multicolumn{2}{l}{\textit{(c) Checkpoint}}\\
  \quad BC-best                      & $-8.2$ \\
  \addlinespace[3pt]
  \multicolumn{2}{l}{\textit{(d) Advantage weighting}}\\
  \quad Uniform weights              & $-6.3$ \\
  \quad Shuffled advantages          & $-6.9$ \\
  \quad Top-$50\%$ filtering         & $-2.0$ \\
  \bottomrule
  \end{tabular}
  \caption{DexArt ablations (MeanFlow, clean demonstrations);
  compact $\Delta$SR values are relative to MF-AWLS ($62.2$). Variants were
  trained in separate batches under the same protocol.}
  \label{tab:ablation}
\end{wraptable}
\paragraph{Frozen vs.\ unfrozen decoder.}
In the decoder-adaptation block, keeping the decoder frozen and steering only
the latent noise performs best among the compared variants.
Unfreezing decoder parameters reduces the gain, and unfreezing both encoder and
decoder performs worse than the BC prior in this setting.
The Q-max steering ablation gives no improvement over BC, suggesting that
the observed gain comes from the advantage-weighted supervised interface rather
than latent steering alone.

\paragraph{Advantage temperature $\beta$.}
With the scale-normalized advantage $\tilde{A}\in[0,1]$, performance is
relatively stable across the tested $\beta$ values, so we use $\beta=3$ without
task-specific retuning.
With the raw advantage $A(\bm{o},\bm{a})$, performance is lower and degrades as $\beta$
grows, consistent with the unbounded advantage scale over-concentrating the
weights (Prop.~\ref{prop:ess}).

\paragraph{Frozen-base checkpoint selection.}
Selecting the frozen base by our rollout-free $S_J$ criterion improves over
selecting the best-BC-score checkpoint in the checkpoint sweep
(Table~\ref{tab:ablation}(c)); both rows apply the same AWLS post-training,
so the difference isolates the choice of frozen base.
This suggests that the highest-BC checkpoint is not always the most steerable
base, and that decoder sensitivity is a useful selection signal.

\paragraph{Isolating advantage weighting.}
Block~(d) keeps the selected frozen base and the shaper fixed and varies only
the weights (Appendix~\ref{app:ablation_weighting}).
Training the same shaper with uniform weights or with shuffled advantages
falls $6.3$ and $6.9$ points below AWLS, back to the level of the BC prior,
so the gain requires the learned advantage signal rather than the shaper or
the checkpoint alone.
Hard top-$50\%$ advantage filtering recovers most of the gain ($-2.0$) but
remains below soft reweighting, which retains all demonstrations and is less
sensitive to imperfect critic rankings.

\paragraph{Residual vs.\ direct noise shaper.}
The residual parameterization outperforms the direct map on clean DexArt (Table~\ref{tab:ablation}(b)), suggesting
that preserving an identity-centered latent interface is beneficial in
practice; the corresponding local geometry and retained action diversity are
examined in Appendix~\ref{app:surrogate}.

\section{Conclusion}
\label{sec:conclusion}

We propose AWLS, a decoder-preserving post-training method for one-step
generative robotic manipulation policies trained from mixed-quality
demonstrations.
AWLS freezes the pretrained decoder and learns a residual noise shaper
with detached, scale-normalized advantage weights, using the critic only to
emphasize higher-quality demonstrated actions rather than to directly optimize
actions.
Across simulation, controlled mixed-quality data, and real-world manipulation
tasks, AWLS improves over the frozen base policy and Q-maximization
baselines while preserving fast single-step inference and limiting drift from
the behavior-cloning prior.

\section{Limitation}
\label{sec:limitations}

AWLS is intended as a conservative post-training method: it can emphasize
better behaviors already present in the offline demonstrations and reachable
through the frozen decoder, but it is not designed to synthesize entirely new
skills absent from the dataset.
Its performance also depends on the offline critic providing a useful relative
ranking of demonstrated actions.
While detached and scale-normalized weighting reduces sensitivity to critic
scale and avoids critic-gradient policy updates, extending AWLS to broader task
families, more diverse visual settings, and less clearly ranked datasets
remains an important direction for future work.

\acknowledgments{This research is supported by National Robotics Programme
(NRP) 2.0 funding initiatives ``Domain-specific Robotics Foundation Models for
Manufacturing (DS-RFM)'' and ``Data Engine Minimum Viable Product (DEM)''.}

\bibliography{example}  % .bib

\begin{thebibliography}{44}
\providecommand{\natexlab}[1]{#1}
\providecommand{\url}[1]{\texttt{#1}}
\expandafter\ifx\csname urlstyle\endcsname\relax
  \providecommand{\doi}[1]{doi: #1}\else
  \providecommand{\doi}{doi: \begingroup \urlstyle{rm}\Url}\fi

\bibitem[Chi et~al.(2023)Chi, Feng, Du, Xu, Cousineau, Burchfiel, and
  Song]{chi2023diffusion}
C.~Chi, S.~Feng, Y.~Du, Z.~Xu, E.~Cousineau, B.~Burchfiel, and S.~Song.
\newblock Diffusion policy: Visuomotor policy learning via action diffusion.
\newblock In \emph{Robotics: Science and Systems}, 2023.

\bibitem[Ze et~al.(2024)Ze, Zhang, Zhang, Hu, Wang, and Xu]{ze2024dp3}
Y.~Ze, G.~Zhang, K.~Zhang, C.~Hu, M.~Wang, and H.~Xu.
\newblock {3D} diffusion policy: Generalizable visuomotor policy learning via
  simple {3D} representations.
\newblock In \emph{Robotics: Science and Systems}, 2024.

\bibitem[Zhao et~al.(2026)Zhao, Wang, Ma, Li, Tay, Ang, and Zhu]{zhao2026fdvla}
R.~Zhao, W.~Wang, Y.~Ma, X.~Li, F.~E.~H. Tay, M.~H. Ang, Jr., and H.~Zhu.
\newblock {FD-VLA}: Force-distilled vision-language-action model for
  contact-rich manipulation.
\newblock In \emph{IEEE International Conference on Robotics and Automation},
  2026.

\bibitem[Brown et~al.(2020)Brown, Goo, and Niekum]{brown2020better}
D.~S. Brown, W.~Goo, and S.~Niekum.
\newblock Better-than-demonstrator imitation learning via automatically-ranked
  demonstrations.
\newblock In \emph{Proceedings of the Conference on Robot Learning}, volume 100
  of \emph{Proceedings of Machine Learning Research}, pages 330--359. PMLR,
  2020.

\bibitem[Mandlekar et~al.(2021)Mandlekar, Xu, Wong, Nasiriany, Wang, Kulkarni,
  Fei-Fei, Savarese, Zhu, and Mart{\'\i}n-Mart{\'\i}n]{mandlekar2021what}
A.~Mandlekar, D.~Xu, J.~Wong, S.~Nasiriany, C.~Wang, R.~Kulkarni, L.~Fei-Fei,
  S.~Savarese, Y.~Zhu, and R.~Mart{\'\i}n-Mart{\'\i}n.
\newblock What matters in learning from offline human demonstrations for robot
  manipulation.
\newblock In \emph{Conference on Robot Learning}, 2021.

\bibitem[Janner et~al.(2022)Janner, Du, Tenenbaum, and
  Levine]{janner2022diffuser}
M.~Janner, Y.~Du, J.~B. Tenenbaum, and S.~Levine.
\newblock Planning with diffusion for flexible behavior synthesis.
\newblock In \emph{International Conference on Machine Learning}, 2022.

\bibitem[Ajay et~al.(2023)Ajay, Du, Gupta, Tenenbaum, Jaakkola, and
  Agrawal]{ajay2023decision}
A.~Ajay, Y.~Du, A.~Gupta, J.~Tenenbaum, T.~Jaakkola, and P.~Agrawal.
\newblock Is conditional generative modeling all you need for decision-making?
\newblock In \emph{International Conference on Learning Representations}, 2023.

\bibitem[Nakamoto et~al.(2024)Nakamoto, Mees, Kumar, and
  Levine]{nakamoto2024vgps}
M.~Nakamoto, O.~Mees, A.~Kumar, and S.~Levine.
\newblock Steering your generalists: Improving robotic foundation models via
  value guidance.
\newblock In \emph{Conference on Robot Learning}, 2024.

\bibitem[Ren et~al.(2025)Ren, Lidard, Ankile, Simeonov, Agrawal, Majumdar,
  Burchfiel, Dai, and Simchowitz]{ren2025dppo}
A.~Z. Ren, J.~Lidard, L.~L. Ankile, A.~Simeonov, P.~Agrawal, A.~Majumdar,
  B.~Burchfiel, H.~Dai, and M.~Simchowitz.
\newblock Diffusion policy policy optimization.
\newblock In \emph{International Conference on Learning Representations}, 2025.

\bibitem[Sobol~Mark et~al.(2025)Sobol~Mark, Gao, Sampaio, Srirama, Sharma,
  Finn, and Kumar]{sobolmark2024parl}
M.~Sobol~Mark, T.~Gao, G.~G. Sampaio, M.~K. Srirama, A.~Sharma, C.~Finn, and
  A.~Kumar.
\newblock Policy agnostic {RL}: Offline {RL} and online {RL} fine-tuning of any
  class and backbone.
\newblock In \emph{International Conference on Learning Representations}, 2025.

\bibitem[Wagenmaker et~al.(2025)Wagenmaker, Nakamoto, Zhang, Park, Yagoub,
  Nagabandi, Gupta, and Levine]{dsrl2025}
A.~Wagenmaker, M.~Nakamoto, Y.~Zhang, S.~Park, W.~Yagoub, A.~Nagabandi,
  A.~Gupta, and S.~Levine.
\newblock Steering your diffusion policy with latent space reinforcement
  learning.
\newblock \emph{arXiv preprint arXiv:2506.15799}, 2025.

\bibitem[Im et~al.(2026)Im, Kolobov, Fu, and Lee]{im2026lps}
H.~Im, A.~Kolobov, J.~Fu, and Y.~Lee.
\newblock Latent policy steering through one-step flow policies.
\newblock \emph{arXiv preprint arXiv:2603.05296}, 2026.

\bibitem[Patil et~al.(2026)Patil, Biza, Weng, Schmeckpeper, Thomason, Zhang,
  Walters, Gopalan, Castro, and Rosen]{patil2026goldenticket}
O.~Patil, O.~Biza, T.~Weng, K.~Schmeckpeper, W.~Thomason, X.~Zhang, R.~Walters,
  N.~Gopalan, S.~Castro, and E.~Rosen.
\newblock You've got a golden ticket: Improving generative robot policies with
  a single noise vector.
\newblock \emph{arXiv preprint arXiv:2603.15757}, 2026.

\bibitem[Lu et~al.(2026)Lu, Qin, Jiang, Wang, Zhang, Liang, Yang, Xu, and
  Zhao]{lu2026unifiednoise}
J.~Lu, X.~Qin, Y.~Jiang, K.~Wang, C.~Zhang, B.~Liang, J.~Yang, M.~Xu, and
  L.~Zhao.
\newblock Unified noise steering for efficient human-guided {VLA} adaptation.
\newblock \emph{arXiv preprint arXiv:2605.10821}, 2026.

\bibitem[Kostrikov et~al.(2022)Kostrikov, Nair, and Levine]{kostrikov2022iql}
I.~Kostrikov, A.~Nair, and S.~Levine.
\newblock Offline reinforcement learning with implicit {Q}-learning.
\newblock In \emph{International Conference on Learning Representations}, 2022.

\bibitem[Peng et~al.(2019)Peng, Kumar, Zhang, and Levine]{peng2019awr}
X.~B. Peng, A.~Kumar, G.~Zhang, and S.~Levine.
\newblock Advantage-weighted regression: Simple and scalable off-policy
  reinforcement learning.
\newblock \emph{arXiv preprint arXiv:1910.00177}, 2019.

\bibitem[Wang et~al.(2025)Wang, Zhao, Teo, Ang, and Zhu]{wang2025tocdp}
W.~Wang, R.~Zhao, T.~J. Teo, M.~H. Ang, Jr., and H.~Zhu.
\newblock Task-guided and object-centric conditioning for effective and
  adaptive diffusion policy.
\newblock In \emph{IEEE/RSJ International Conference on Intelligent Robots and
  Systems}, pages 16746--16753, 2025.
\newblock \doi{10.1109/IROS60139.2025.11247520}.

\bibitem[Wang et~al.(2026)Wang, Zhao, Teo, Ang, and Zhu]{wang2026dream}
W.~Wang, R.~Zhao, T.~J. Teo, M.~H. Ang, Jr., and H.~Zhu.
\newblock From dream to action: Hierarchical policy learning with {3D} world
  imagination for robotic manipulation.
\newblock In \emph{IEEE International Conference on Robotics and Automation},
  pages 1155--1162, 2026.

\bibitem[Yang et~al.(2025)Yang, Cheng, Zhu, Li, Tao, and
  Mao]{yang2025hfdteacher}
Z.~Yang, A.~Cheng, H.~Zhu, T.~Li, P.~Y. Tao, and K.~Mao.
\newblock {HFD-Teacher}: High-frequency depth distillation from depth
  foundation models for enhanced depth completion.
\newblock In \emph{Proceedings of the IEEE/CVF International Conference on
  Computer Vision}, pages 8994--9003, October 2025.

\bibitem[Prasad et~al.(2024)Prasad, Lin, Wu, Zhou, and
  Bohg]{prasad2024consistencypolicy}
A.~Prasad, K.~Lin, J.~Wu, L.~Zhou, and J.~Bohg.
\newblock Consistency policy: Accelerated visuomotor policies via consistency
  distillation.
\newblock In \emph{Proceedings of Robotics: Science and Systems}, Delft,
  Netherlands, July 2024.

\bibitem[Wang et~al.(2025)Wang, Li, Mandlekar, Xu, Fan, Narang, Fan, Zhu,
  Balaji, Zhou, Liu, and Zeng]{wang2025onedp}
Z.~Wang, Z.~Li, A.~Mandlekar, Z.~Xu, J.~Fan, Y.~Narang, L.~Fan, Y.~Zhu,
  Y.~Balaji, M.~Zhou, M.-Y. Liu, and Y.~Zeng.
\newblock One-step diffusion policy: Fast visuomotor policies via diffusion
  distillation.
\newblock In \emph{International Conference on Machine Learning}, 2025.

\bibitem[Wang et~al.(2023)Wang, Hunt, and Zhou]{wang2023diffusionql}
Z.~Wang, J.~J. Hunt, and M.~Zhou.
\newblock Diffusion policies as an expressive policy class for offline
  reinforcement learning.
\newblock In \emph{International Conference on Learning Representations}, 2023.

\bibitem[Hansen-Estruch et~al.(2023)Hansen-Estruch, Kostrikov, Janner, Kuba,
  and Levine]{hansenestruch2023idql}
P.~Hansen-Estruch, I.~Kostrikov, M.~Janner, J.~G. Kuba, and S.~Levine.
\newblock {IDQL}: Implicit {Q}-learning as an actor-critic method with
  diffusion policies.
\newblock \emph{arXiv preprint arXiv:2304.10573}, 2023.

\bibitem[Park et~al.(2025)Park, Li, and Levine]{fql2025}
S.~Park, Q.~Li, and S.~Levine.
\newblock Flow {Q}-learning.
\newblock In \emph{International Conference on Machine Learning}, 2025.

\bibitem[Zhang et~al.(2025)Zhang, Xue, Niu, Chen, Yang, Li, and
  Liu]{zhang2025gmpo}
J.~Zhang, R.~Xue, Y.~Niu, Y.~Chen, J.~Yang, H.~Li, and Y.~Liu.
\newblock Revisiting generative policies: A simpler reinforcement learning
  algorithmic perspective.
\newblock In \emph{International Conference on Learning Representations}, 2025.

\bibitem[McAllister et~al.(2026)McAllister, Ge, Yi, Kim, Weber, Choi, Feng, and
  Kanazawa]{mcallister2026fpo}
D.~McAllister, S.~Ge, B.~Yi, C.~M. Kim, E.~Weber, H.~Choi, H.~Feng, and
  A.~Kanazawa.
\newblock Flow matching policy gradients.
\newblock In \emph{International Conference on Learning Representations}, 2026.

\bibitem[Domingo-Enrich et~al.(2025)Domingo-Enrich, Drozdzal, Karrer, and
  Chen]{domingo2025adjoint}
C.~Domingo-Enrich, M.~Drozdzal, B.~Karrer, and R.~T.~Q. Chen.
\newblock Adjoint matching: Fine-tuning flow and diffusion generative models
  with memoryless stochastic optimal control.
\newblock In \emph{International Conference on Learning Representations}, 2025.
\newblock Spotlight.

\bibitem[Zheng et~al.(2026)Zheng, Chen, Ye, Wang, Zhang, Jiang, Su, Ermon, Zhu,
  and Liu]{zheng2026diffusionnft}
K.~Zheng, H.~Chen, H.~Ye, H.~Wang, Q.~Zhang, K.~Jiang, H.~Su, S.~Ermon, J.~Zhu,
  and M.-Y. Liu.
\newblock {DiffusionNFT}: Online diffusion reinforcement with forward process.
\newblock In \emph{International Conference on Learning Representations}, 2026.
\newblock Oral.

\bibitem[Kumar et~al.(2020)Kumar, Zhou, Tucker, and Levine]{kumar2020cql}
A.~Kumar, A.~Zhou, G.~Tucker, and S.~Levine.
\newblock Conservative {Q}-learning for offline reinforcement learning.
\newblock In \emph{Advances in Neural Information Processing Systems}, 2020.

\bibitem[Du and Song(2025)]{du2025dynaguide}
M.~Du and S.~Song.
\newblock {DynaGuide}: Steering diffusion policies with active dynamic
  guidance.
\newblock In \emph{Advances in Neural Information Processing Systems}, 2025.

\bibitem[Zhou et~al.(2021)Zhou, Bajracharya, and Held]{zhou2021plas}
W.~Zhou, S.~Bajracharya, and D.~Held.
\newblock {PLAS}: Latent action space for offline reinforcement learning.
\newblock In \emph{Proceedings of the Conference on Robot Learning}, 2021.

\bibitem[Chen et~al.(2022)Chen, Ghadirzadeh, Yu, Wang, Gao, Li, Liang, Finn,
  and Zhang]{chen2022lapo}
X.~Chen, A.~Ghadirzadeh, T.~Yu, J.~Wang, A.~Y. Gao, W.~Li, B.~Liang, C.~Finn,
  and C.~Zhang.
\newblock {LAPO}: Latent-variable advantage-weighted policy optimization for
  offline reinforcement learning.
\newblock In \emph{Advances in Neural Information Processing Systems}, 2022.

\bibitem[Ki et~al.(2025)Ki, Oh, Shim, and Lee]{ki2025pg}
D.~Ki, J.~Oh, S.-W. Shim, and B.-J. Lee.
\newblock Prior-guided diffusion planning for offline reinforcement learning.
\newblock In \emph{Advances in Neural Information Processing Systems}, 2025.

\bibitem[Nair et~al.(2020)Nair, Gupta, Dalal, and Levine]{nair2020awac}
A.~Nair, A.~Gupta, M.~Dalal, and S.~Levine.
\newblock {AWAC}: Accelerating online reinforcement learning with offline
  datasets.
\newblock \emph{arXiv preprint arXiv:2006.09359}, 2020.

\bibitem[Wang et~al.(2020)Wang, Novikov, Zolna, Merel, Springenberg, Reed,
  Shahriari, Siegel, Gulcehre, Heess, and de~Freitas]{wang2020crr}
Z.~Wang, A.~Novikov, K.~Zolna, J.~S. Merel, J.~T. Springenberg, S.~E. Reed,
  B.~Shahriari, N.~Siegel, C.~Gulcehre, N.~Heess, and N.~de~Freitas.
\newblock Critic regularized regression.
\newblock In \emph{Advances in Neural Information Processing Systems}, 2020.

\bibitem[Lipman et~al.(2023)Lipman, Chen, Ben-Hamu, Nickel, and
  Le]{lipman2023flow}
Y.~Lipman, R.~T.~Q. Chen, H.~Ben-Hamu, M.~Nickel, and M.~Le.
\newblock Flow matching for generative modeling.
\newblock In \emph{International Conference on Learning Representations}, 2023.

\bibitem[Geng et~al.(2025)Geng, Deng, Bai, Kolter, and He]{geng2025meanflow}
Z.~Geng, M.~Deng, X.~Bai, J.~Z. Kolter, and K.~He.
\newblock Mean flows for one-step generative modeling.
\newblock In \emph{Advances in Neural Information Processing Systems}, 2025.
\newblock Oral.

\bibitem[Sheng et~al.(2026)Sheng, Wang, Li, and Liu]{sheng2026mp1}
J.~Sheng, Z.~Wang, P.~Li, and M.~Liu.
\newblock {MP1}: {MeanFlow} tames policy learning in 1-step for robotic
  manipulation.
\newblock In \emph{AAAI Conference on Artificial Intelligence}, 2026.

\bibitem[Zhan et~al.(2026)Zhan, Tao, Wang, Wang, Chen, Li, Li, Tomizuka, and
  Li]{zhan2026mvp}
G.~Zhan, L.~Tao, P.~Wang, Y.~Wang, Y.~Chen, Y.~Li, H.~Li, M.~Tomizuka, and
  S.~E. Li.
\newblock Mean flow policy with instantaneous velocity constraint for one-step
  action generation.
\newblock In \emph{International Conference on Learning Representations}, 2026.
\newblock Oral.

\bibitem[Grathwohl et~al.(2019)Grathwohl, Chen, Bettencourt, Sutskever, and
  Duvenaud]{grathwohl2019ffjord}
W.~Grathwohl, R.~T.~Q. Chen, J.~Bettencourt, I.~Sutskever, and D.~Duvenaud.
\newblock {FFJORD}: Free-form continuous dynamics for scalable reversible
  generative models.
\newblock In \emph{International Conference on Learning Representations}, 2019.

\bibitem[Rajeswaran et~al.(2018)Rajeswaran, Kumar, Gupta, Vezzani, Schulman,
  Todorov, and Levine]{rajeswaran2018learning}
A.~Rajeswaran, V.~Kumar, A.~Gupta, G.~Vezzani, J.~Schulman, E.~Todorov, and
  S.~Levine.
\newblock Learning complex dexterous manipulation with deep reinforcement
  learning and demonstrations.
\newblock In \emph{Robotics: Science and Systems}, 2018.
\newblock \doi{10.15607/RSS.2018.XIV.049}.

\bibitem[Bao et~al.(2023)Bao, Xu, Qin, and Wang]{bao2023dexart}
C.~Bao, H.~Xu, Y.~Qin, and X.~Wang.
\newblock {DexArt}: Benchmarking generalizable dexterous manipulation with
  articulated objects.
\newblock In \emph{Proceedings of the IEEE/CVF Conference on Computer Vision
  and Pattern Recognition}, 2023.

\bibitem[Mu et~al.(2025)Mu, Chen, Chen, Peng, Lan, Gao, Liang, Yu, Zou, Xu,
  Lin, Xie, Ding, and Luo]{mu2025robotwin}
Y.~Mu, T.~Chen, Z.~Chen, S.~Peng, Z.~Lan, Z.~Gao, Z.~Liang, Q.~Yu, Y.~Zou,
  M.~Xu, L.~Lin, Z.~Xie, M.~Ding, and P.~Luo.
\newblock {RoboTwin}: Dual-arm robot benchmark with generative digital twins.
\newblock In \emph{Proceedings of the IEEE/CVF Conference on Computer Vision
  and Pattern Recognition}, 2025.

\bibitem[Zhang et~al.(2025)Zhang, Liu, Fan, Liu, Zeng, and
  Liu]{zhang2025flowpolicy}
Q.~Zhang, Z.~Liu, H.~Fan, G.~Liu, B.~Zeng, and S.~Liu.
\newblock {FlowPolicy}: Enabling fast and robust {3D} flow-based policy via
  consistency flow matching for robot manipulation.
\newblock In \emph{AAAI Conference on Artificial Intelligence}, volume~39,
  pages 14754--14762, 2025.

\end{thebibliography}

\clearpage
\appendix
\label{pg:app-first}%
\begin{center}
  {\Large\bfseries Appendix\par}
\end{center}
\vspace{4pt}

\section{One-Step Decoder Background}
\label{app:meanflow}

AWLS treats the one-step decoder $G_\theta$ as a frozen black box: the only
property the rest of the paper relies on is that $G_\theta(\bm{o},\bm{\epsilon})$
maps an observation $\bm{o}$ and a Gaussian latent
$\bm{\epsilon}\sim\mathcal{N}(0,I)$ (of the same shape as the action chunk) to an
executable action chunk in a single network evaluation.
Both flow-based backbones we use, MeanFlow and FlowPolicy, provide this
interface; we refer to the original papers for their training objectives.

\paragraph{A common one-step interface.}
Although MeanFlow and FlowPolicy differ in parameterization (interval-average velocity
$\bm{u}_\theta$ vs.\ instantaneous velocity $\bm{v}_\theta$) and training signal (an
exact JVP-based stop-gradient identity vs.\ a finite-difference EMA consistency
objective), they yield the \emph{same} interface for downstream use. In both cases the
trained network collapses, at its boundary evaluation (the full interval
$r=0,t=1$ for MeanFlow; $t=1$ for FlowPolicy), to a single frozen map
\begin{equation}
  \label{eq:onestep_interface}
  G_\theta:\;(\bm{o},\bm{\epsilon})\;\longmapsto\;\bm{x}_0,
  \qquad \bm{\epsilon}\sim\mathcal{N}(0,I),
\end{equation}
which is the interface of Eq.~\eqref{eq:onestep_decoder}; AWLS is therefore
agnostic to which backbone---MeanFlow or FlowPolicy---produced the decoder.

\section{Critic and Advantage Estimation}
\label{app:critic}

Advantages are estimated following IQL.
A 2-ensemble of Q-functions yields a pessimistic estimate
$Q_{\mathrm{pess}}(\bm{o},\bm{a})=\min_k Q_{\psi_k}(\bm{o},\bm{a})$.
The value baseline is trained with expectile regression on this pessimistic
estimate:
\begin{equation}
  \label{eq:iql_value}
  \mathcal{L}_V(\nu)
    = \mathbb{E}_{(\bm{o},\bm{a})\sim\mathcal{D}}\!\Big[
        \mathcal{L}_2^\tau\!\big(
          Q_{\mathrm{pess}}(\bm{o},\bm{a}) - V_\nu(\bm{o})
        \big)
      \Big],
  \qquad \tau > 0.5,
\end{equation}
where $\mathcal{L}_2^\tau(u)=|\tau-\mathbb{1}\{u<0\}|u^2$.
The detached advantage used by AWLS is
$A(\bm{o},\bm{a})=\operatorname{sg}[Q_{\mathrm{pess}}(\bm{o},\bm{a})-V_\nu(\bm{o})]$
(Eq.~\eqref{eq:awls_advantage}).
We use the goal-conditioned binary reward $r=\mathbb{1}[\,\text{task-success
predicate at step }t\,]$: the reward is $1$ on the transition that first
satisfies the per-task success predicate already used by the simulator to score
rollouts (the same predicate that defines the reported success rate---e.g.\
the object reaching its target pose/region for the DexArt and Adroit tasks, or
the ordered stage being completed for RoboTwin), and $0$ otherwise, with the
episode terminating ($d=1$) on success. Using the evaluation predicate as the
reward keeps the offline advantage aligned with the metric we report.
We train the Q-functions with Bellman regression---augmented with a
conservative Q-learning penalty ($\alpha_{\mathrm{CQL}}=0.5$)
that suppresses over-estimation on out-of-support noise---using target
$y = r + \gamma(1-d)\,V_{\bar{\nu}}(\bm{o}')$.

\paragraph{Does the critic rank clean demonstrations above noisy ones?}
The mixed-quality analysis in the main paper attributes AWLS's preference for
cleaner demonstrated behavior to the critic assigning larger advantage weights
to clean actions. We test this premise by measuring whether the offline
advantage separates clean from noisy demonstrated actions. On each of the
four $50/50$ clean/noisy mixed DexArt datasets we split the demonstrated pairs
$(\bm{o},\bm{a})$ into the clean subset and the Ornstein--Uhlenbeck--perturbed noisy
subset (per-task perturbation magnitude $\sigma$), score every pair with the detached
advantage $A$ (Eq.~\ref{eq:awls_advantage}) and its scale-normalized form
$\tilde{A}\in[0,1]$ (Eq.~\ref{eq:awls_norm}, the quantity entering the weight
$w=\exp(\beta\tilde{A})$ at $\beta=3$), and compare the two subsets in
Table~\ref{tab:adv_separation} and Figure~\ref{fig:adv_hist}.

Three observations support the claim. (i) On every task the critic
assigns higher advantage to clean than to noisy actions, so clean actions receive
$2.1$--$2.6\times$ ($\approx2.3\times$ on average) the weight of noisy ones---all
inside the bounded-weight regime $w\in[1,e^{\beta}]$ of Prop.~\ref{prop:ess}, so the
separation never collapses the effective sample size. (ii) Advantage is positively
correlated with the demonstrator's return on every task ($\rho(A,R)>0$), i.e.\ the
critic ranks demonstrated actions by realized quality rather than arbitrarily.
(iii) The separation is strong and roughly uniform across tasks (AUROC
$0.85$--$0.92$), as expected: although the per-task OU scale $\sigma$ spans
$0.15$--$0.6$, it is tuned per task so that every noisy pool reaches a similar low
quality (expert rollout success below $20\%$, by construction of the mixed-quality
datasets), so the clean/noisy quality gap---and hence the critic's ability to
separate the two---is comparable across tasks. The small $\sigma$ for
\textsc{Bucket} reflects that task's higher \emph{sensitivity} to action noise, not
a milder quality gap. The aggregate weighting is consistent with the residual
shaper outperforming the direct-interface variant on clean DexArt, where the
direct variant is $2.5$ points worse (Table~\ref{tab:ablation}(b)).

Two clarifications keep this evidence honest. The clean/noisy split is a
provenance label that the critic never observes---it is trained only on the
goal-conditioned success reward---so a high clean-vs-noisy AUROC is \emph{not}
circular: the critic \emph{recovers} a demonstration-quality distinction it was
never given. The correlation $\rho(A,R)$, in contrast, is partly definitional,
since the reward is the deployment success predicate; a positive $\rho$ therefore
confirms that the advantage faithfully encodes the reported metric rather than
establishing an independent notion of quality. The real-world study
(Appendix~\ref{app:realworld}) is intended to provide that independent axis, and
the $50/50$ mix used here is the most balanced (and hardest-to-disentangle)
setting.

\begin{table}[t]
  \centering
  \caption{\textbf{The IQL critic ranks clean demonstrations above noisy ones}
  (four $50/50$ clean/noisy mixed DexArt datasets, MeanFlow backbone; advantages
  detached). For each task we report the mean scale-normalized advantage
  $\tilde{A}\in[0,1]$ (mean$\pm$std over three seeds) on the clean vs.\ OU-perturbed noisy subset,
  the induced weight ratio
  $w_{\text{clean}}/w_{\text{noisy}}=\exp\!\big(\beta(\bar{\tilde{A}}_{\text{clean}}
  -\bar{\tilde{A}}_{\text{noisy}})\big)$ at $\beta=3$, the AUROC (empirical
  rank-order) of separating clean from noisy actions by advantage, and the Spearman
  correlation $\rho(A,R)$ between
  advantage and demonstrator return. Clean actions receive higher advantage on every
  task; the separation is strong and roughly uniform across tasks because each
  per-task $\sigma$ is tuned to drive its noisy pool to a similar low quality
  ($<20\%$ success), so $\sigma$ reflects task sensitivity to noise, not the quality
  gap. Raw $A$ is in return units (goal-conditioned,
  $\gamma=0.99$); $\tilde{A}$ is its per-task min--max normalization, so the $[0,1]$
  scale is the per-task advantage range by construction. The bottom row is the
  per-task mean of each column.}
  \label{tab:adv_separation}
  \begin{tabular}{lcccccc}
    \toprule
    Task & $\sigma_{\text{OU}}$ & $\tilde{A}_{\text{clean}}$ & $\tilde{A}_{\text{noisy}}$
      & $w_{\text{clean}}/w_{\text{noisy}}$ & AUROC & $\rho(A,R)$ \\
    \midrule
    Laptop & $0.4$  & $0.66\pm0.15$ & $0.39\pm0.17$ & $2.26$ & $0.89$ & $0.47$ \\
    Faucet & $0.6$  & $0.63\pm0.15$ & $0.39\pm0.17$ & $2.09$ & $0.86$ & $0.49$ \\
    Toilet & $0.6$  & $0.66\pm0.14$ & $0.35\pm0.17$ & $2.55$ & $0.92$ & $0.56$ \\
    Bucket & $0.15$ & $0.65\pm0.15$ & $0.40\pm0.18$ & $2.10$ & $0.85$ & $0.46$ \\
    \midrule
    DexArt avg. & -- & $0.65$ & $0.38$ & $2.25$ & $0.88$ & $0.49$ \\
    \bottomrule
  \end{tabular}
\end{table}

\begin{figure}[t]
  \centering
  \includegraphics[width=\linewidth]{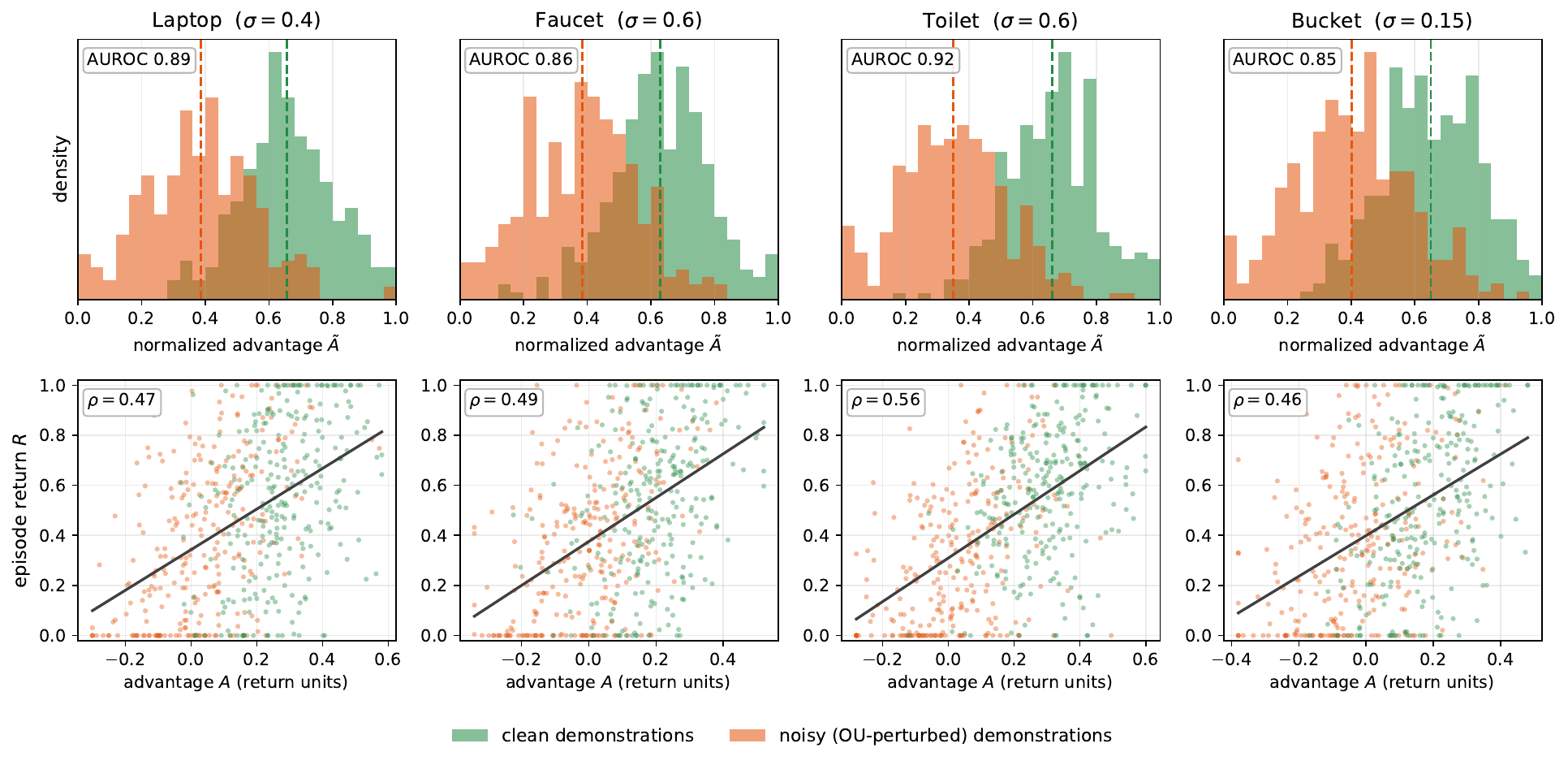}
  \caption{\textbf{The IQL critic ranks clean demonstrations above noisy ones}
  (four $50/50$ mixed DexArt datasets, MeanFlow backbone). \textbf{Top:} per-task
  histograms of the scale-normalized advantage $\tilde{A}\in[0,1]$ on clean (green)
  vs.\ OU-perturbed noisy (orange) demonstrated actions; dashed lines mark the
  subset means. Clean actions carry systematically higher advantage on every task,
  and advantage separates the two subsets with empirical rank-order AUROC
  $0.85$--$0.92$---even though clean/noisy provenance is a label the critic never
  observes during training. \textbf{Bottom:} the detached advantage $A$ (return
  units) vs.\ the demonstrator's episode return $R$, with the linear trend;
  advantage is positively correlated with return on every task (Spearman
  $\rho=0.46$--$0.56$). Summary statistics match Table~\ref{tab:adv_separation};
  histograms and scatters pool all demonstrated pairs over the three seeds.}
  \label{fig:adv_hist}
\end{figure}

\paragraph{Do advantages separate quality differences within successful
real-world demonstrations?}
The clean-vs-noisy separation above concerns synthetic corruption. To test
whether the binary-reward critic also captures local quality variation inside
successful human demonstrations, we reviewed every step of the $70$ real-world
Tube Insertion demonstrations (Appendix~\ref{app:realworld}), manually flagged
problematic segments, and assigned each flagged step to exactly one category:
\emph{detour}, \emph{hesitation}, or \emph{jerk/abrupt motion}. Within each
trajectory, steps whose advantage $A_t=A(\bm{o}_t,\bm{a}_t)$ falls below the
trajectory median form the $\mathrm{Low}$-$A$ set (fixed tie-breaking).
Table~\ref{tab:adv_alignment} reports, for each category, the share of its
steps that fall in $\mathrm{Low}$-$A$; a share above the $50\%$ chance level
indicates segments the critic identifies as lower quality. Detours ($64.0\%$)
and hesitation ($71.3\%$) are strongly enriched, whereas jerk/abrupt motion
stays at chance ($50.7\%$): the binary-reward critic captures progress-related
inefficiency well, while fine-grained smoothness is largely invisible to it.

\begin{table}[t]
  \centering
  \small
  \caption{Advantage--quality alignment on $70$ real-world Tube Insertion
  demonstrations: share of manually flagged steps in each category that fall in
  the low-advantage half of their trajectory ($50\%$ = chance).}
  \label{tab:adv_alignment}
  \begin{tabular}{lccc}
    \toprule
    Category & Detour & Hesitation & Jerk/abrupt \\
    \midrule
    In $\mathrm{Low}$-$A$ & $626/978$ ($\mathbf{64.0\%}$) & $345/484$ ($\mathbf{71.3\%}$) & $149/294$ ($50.7\%$) \\
    \bottomrule
  \end{tabular}
\end{table}

\section{Reconstruction Objective and Residual-Map Geometry}
\label{app:surrogate}

This appendix distinguishes the reconstruction objective used by AWLS from
marginal action likelihood, then establishes a conditional result for the
local geometry of the residual shaper. We evaluate retained action diversity
separately rather than infer it from the residual form.

\paragraph{Reconstruction and likelihood.}
Write the composed policy as
$f_\phi(\bm{o},\bm{\epsilon})
 =G_\theta(\bm{o},T_\phi(\bm{o},\bm{\epsilon}))$.
For an auxiliary Gaussian observation model
$p_\sigma(\bm{a}\mid\bm{o},\bm{\epsilon})
 =\mathcal{N}(\bm{a};f_\phi(\bm{o},\bm{\epsilon}),\sigma^2I)$
with fixed $\sigma^2>0$, the expected conditional negative log-likelihood is
\begin{equation}
  \label{eq:conditional_reconstruction_nll}
  \mathbb{E}_{\bm{\epsilon}}\!\left[
    -\log p_\sigma(\bm{a}\mid\bm{o},\bm{\epsilon})\right]
  =\frac{1}{2\sigma^2}\,
    \ell_{\mathrm{MSE}}^\phi(\bm{o},\bm{a})+\text{const}.
\end{equation}
This is a conditional Gaussian interpretation of reconstruction;
$\sigma$ is an analytical device and AWLS adds no Gaussian action noise at
deployment. For the auxiliary marginal
$p_\sigma(\bm{a}\mid\bm{o})
 =\mathbb{E}_{\bm{\epsilon}}[p_\sigma(\bm{a}\mid\bm{o},\bm{\epsilon})]$,
Jensen's inequality gives
\begin{equation}
  \label{eq:reconstruction_jensen}
  -\log p_\sigma(\bm{a}\mid\bm{o})
  \leq \mathbb{E}_{\bm{\epsilon}}\!\left[
      -\log p_\sigma(\bm{a}\mid\bm{o},\bm{\epsilon})\right].
\end{equation}
Positive detached weights preserve this inequality after averaging over
observations and demonstrated actions, but do not ensure that the bound is
tight or that optimizing it matches a marginal-likelihood update.

The deterministic pushforward policy used at deployment has a different
likelihood. If both maps are dimension-matched diffeomorphisms, let
$\bm{\epsilon}_{\bm{a}}=f_\phi(\bm{o},\cdot)^{-1}(\bm{a})$ and
$\bm{z}_{\bm{a}}=T_\phi(\bm{o},\bm{\epsilon}_{\bm{a}})$.
The change-of-variables formula~\cite{grathwohl2019ffjord} gives
\begin{equation}
  \label{eq:pushforward_nll}
  -\log\pi_\phi(\bm{a}\mid\bm{o})
  =-\log p_{\bm{\epsilon}}(\bm{\epsilon}_{\bm{a}})
    +h_G(\bm{o},\bm{z}_{\bm{a}})
    +h_T(\bm{o},\bm{\epsilon}_{\bm{a}}),
\end{equation}
where $p_{\bm{\epsilon}}$ is the standard Gaussian density,
$h_G=\log|\det(\partial G_\theta/\partial\bm{z})|$, and
$h_T=\log|\det(\partial T_\phi/\partial\bm{\epsilon})|$.
These terms are evaluated at the demonstrated action's inverse latent,
whereas reconstruction averages over independent latent draws. AWLS therefore
uses an AWR-inspired supervised surrogate; we do not assert an MSE--marginal
likelihood identity or a bound on their gradient discrepancy.

\paragraph{Conditional local geometry.}
For a fixed observation, the residual map
$T_\phi(\bm{o},\bm{\epsilon})
 =\bm{\epsilon}+\Delta_\phi(\bm{o},\bm{\epsilon})$
has Jacobian $J_T=I+J_\Delta$, where
$J_\Delta=\partial\Delta_\phi/\partial\bm{\epsilon}$.
The following result controls the local deformation of this map.
\begin{proposition}[Local conditioning of the residual map]
  \label{prop:surrogate}
  Let $\Delta_\phi(\bm{o},\cdot):\mathbb{R}^d\to\mathbb{R}^d$ be
  continuously differentiable. At any state--noise pair satisfying
  $\|J_\Delta\|_{\mathrm{op}}\leq L_v<1$, the shaper Jacobian is
  nonsingular and obeys
  \begin{equation}
    \label{eq:residual_conditioning}
    1-L_v\leq\sigma_{\min}(J_T)
      \leq\sigma_{\max}(J_T)\leq1+L_v,
    \qquad
    \kappa_2(J_T)\leq\frac{1+L_v}{1-L_v}.
  \end{equation}
  Its local log-determinant satisfies
  \begin{equation}
    \label{eq:residual_logdet_bound}
    d\log(1-L_v)\leq h_T\leq d\log(1+L_v).
  \end{equation}
\end{proposition}

\begin{proof}
For any vector $\bm{v}$, the triangle and reverse triangle inequalities give
\[
  (1-L_v)\|\bm{v}\|_2
  \leq\|(I+J_\Delta)\bm{v}\|_2
  \leq(1+L_v)\|\bm{v}\|_2.
\]
Taking the infimum and supremum over unit vectors yields the singular-value
bounds and hence the condition-number bound. Since $1-L_v>0$, $J_T$ is
nonsingular. The absolute determinant is the product of its singular values;
taking logarithms gives Eq.~\eqref{eq:residual_logdet_bound}.
\end{proof}

\begin{remark}[Scope and measured evidence]
The proposition is a pointwise geometric result. If its condition holds
uniformly over all $\bm{\epsilon}\in\mathbb{R}^d$ for a fixed observation,
then $\Delta_\phi$ is $L_v$-Lipschitz and $T_\phi$ is globally invertible:
for any target $\bm{z}$, the inverse is the unique fixed point of the
contraction $\bm{\epsilon}\mapsto\bm{z}-\Delta_\phi(\bm{o},\bm{\epsilon})$.
AWLS does not enforce this uniform condition. Table~\ref{tab:jacobian_probe}
reports that $\|J_\Delta\|_{\mathrm{op}}<1$ at every evaluated state--noise
sample, supporting nonsingular local geometry at those points rather than
certifying global invertibility.

Local Jacobian control also does not bound residual displacement: a constant
correction can have zero Jacobian and still shift the latent input
substantially. We therefore measure displacement separately, alongside the
Jacobian diagnostics. Neither measurement establishes a bound on the
MSE--likelihood discrepancy or guarantees that generated actions stay within
data support; action diversity and task performance are assessed empirically.
Table~\ref{tab:surrogate} summarizes these conditional properties for the
Gaussian, direct, and residual shaper structures, and
Figure~\ref{fig:residual_direct} compares the residual and direct shapers
empirically on clean DexArt (the comparison behind Table~\ref{tab:ablation}(b)).
\end{remark}

\begin{table}[!t]
\centering
\small
\caption{Noise-shaper structure and conditional geometric properties.
  The residual bounds require $\|J_\Delta\|_{\mathrm{op}}\leq L_v<1$.}
\label{tab:surrogate}
\begin{tabular}{lccc}
\toprule
 & Gaussian & Direct NN & Residual \\
\midrule
$J_T$ & $\operatorname{diag}(\bm{\sigma}_\phi)$ & $J_{f_\phi}$ & $I+J_\Delta$ \\
Nonsingularity & $\min_i\sigma_i>0$ & Requires a condition & If $L_v<1$ \\
$\kappa_2(J_T)$ & $\max_i\sigma_i/\min_i\sigma_i$ & Unconstrained & $\leq(1+L_v)/(1-L_v)$ \\
\bottomrule
\end{tabular}
\end{table}

\begin{figure}[!t]
\centering
\includegraphics[width=0.62\textwidth]{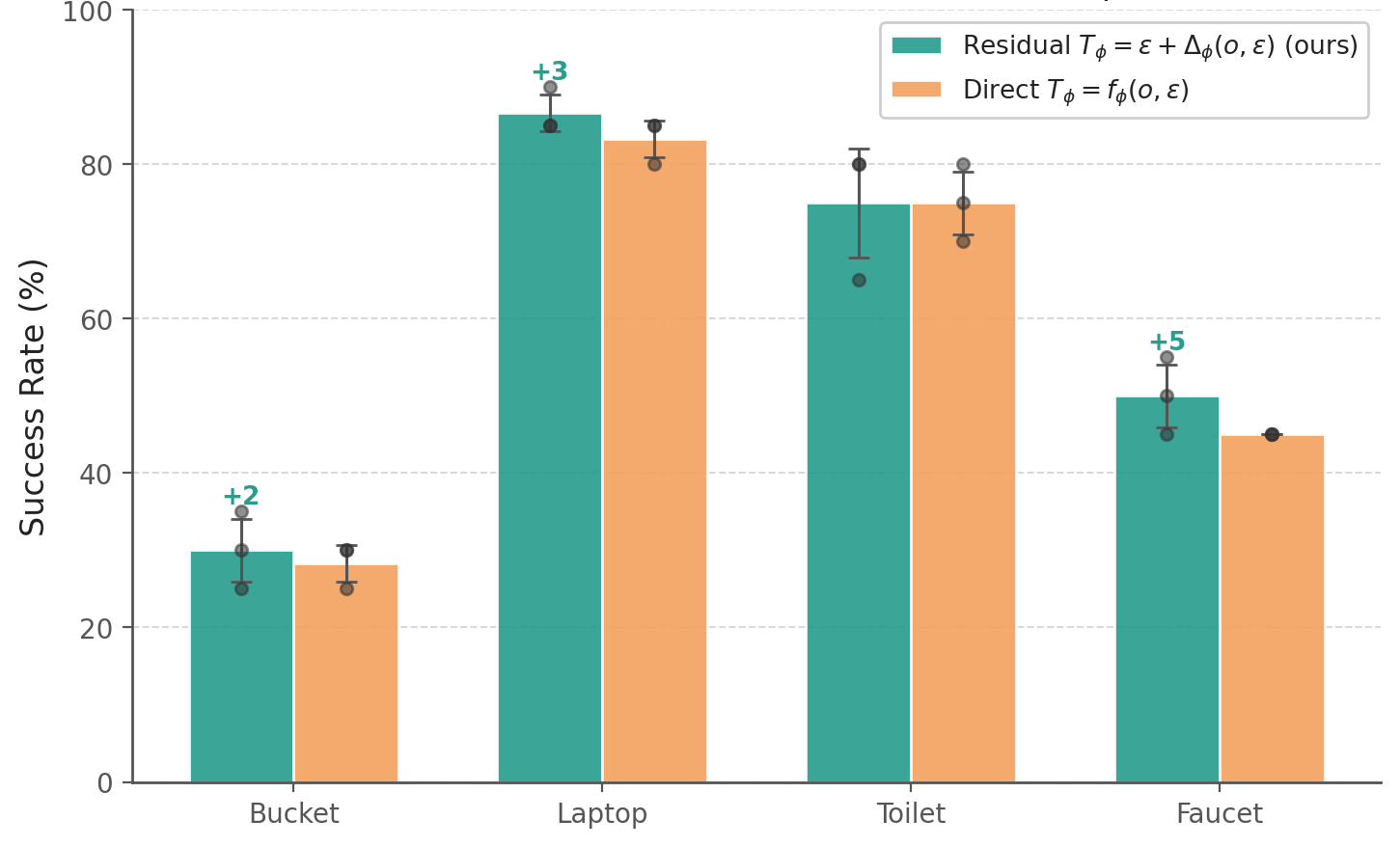}
\caption{Residual vs.\ direct noise shaper on clean DexArt (ablation batch,
  mean$\pm$std over three seeds, best checkpoint). The residual
  parameterization $T_\phi=\bm{\epsilon}+\Delta_\phi$ matches or exceeds the
  unconstrained direct map $T_\phi=f_\phi$ on every task. This comparison
  supports the residual parameterization empirically; its local geometric
  measurements are reported separately in Table~\ref{tab:jacobian_probe}.}
\label{fig:residual_direct}
\end{figure}

\paragraph{Objective surrogate and latent diversity.}
Because the marginal action likelihood of the composed policy is not directly
tractable, AWLS uses advantage-weighted reconstruction as an AWR-inspired
surrogate with the conditional Gaussian interpretation above. Regressing
independent noise draws toward demonstrated targets can reduce latent
diversity by mapping different draws to similar actions.
Table~\ref{tab:diversity} reports the
relative action diversity $D_{\mathrm{rel}}$, the mean pairwise spread of
decoded actions across latent samples at matched held-out states (one shared
noise book), normalized by the spread of the BC backbone from the same run
($1$: base diversity, $0$: collapse), together with success rate. All entries
aggregate two-seed, newly retrained runs over the four DexArt tasks; for each
task and seed the post-training variants share the same frozen base
checkpoint. The residual shaper retains $0.457$ of the base policy's action
spread, more than any of the other tested shapers, and reaches the highest
success rate among these variants. These results show partial retention of
diversity rather than preservation of the original conditional action
distribution. An unconstrained direct map and a Gaussian reparameterization
collapse diversity almost completely; entropy regularization partly restores
it at the cost of an additional objective and coefficient.

\begin{table}[t]
  \centering
  \small
  \caption{Latent diversity and success of noise-shaper variants on DexArt
  (two-seed matched reruns; shared frozen base per task and seed).
  $D_{\mathrm{rel}}$ is the mean pairwise action spread across latent samples
  relative to the BC backbone ($1$: base, $0$: collapse).}
  \label{tab:diversity}
  \begin{tabular}{lccccc}
    \toprule
    & BC/Base & Direct & Gaussian & Gaussian $+$ reg. & \textbf{Residual} \\
    \midrule
    $D_{\mathrm{rel}}$ $\uparrow$ & $1.000$ & $5.4{\times}10^{-4}$ & $5.2{\times}10^{-5}$ & $0.380$ & $\mathbf{0.457}$ \\
    SR (\%) $\uparrow$ & $55.2$ & $59.8$ & $59.5$ & $60.2$ & $\mathbf{61.9}$ \\
    \bottomrule
  \end{tabular}
\end{table}

\section{Proof of Proposition~\ref{prop:ess} (Bounded Weight Concentration)}
\label{app:ess}

We restate the claim of Proposition~\ref{prop:ess}: for a minibatch of size
$n$ the AWLS weights satisfy $w_{\max}/w_{\min}\leq R$ with $R=e^\beta$, and the
effective sample size obeys
$\mathrm{ESS}=(\sum_i w_i)^2/\sum_i w_i^2 \geq 4nR/(1+R)^2$, which evaluates to
$\mathrm{ESS}\geq 46.3$ for $\beta=3$ and $n=256$.

\begin{proof}[Proof of Proposition~\ref{prop:ess}]
\textbf{Step 1 (weight ratio).}
Each AWLS weight is $w_i=\exp(\beta\,\tilde{A}_i)$ with the per-batch normalized
advantage $\tilde{A}_i\in[0,1]$ from Eq.~\eqref{eq:awls_norm}. Since
$x\mapsto e^{\beta x}$ is increasing for $\beta>0$,
$w_i\in[e^0,e^\beta]=[1,e^\beta]$, hence $w_{\min}\geq 1$, $w_{\max}\leq e^\beta$,
and $w_{\max}/w_{\min}\leq e^\beta=:R$. This holds for any minibatch and any
critic state, because $\tilde{A}_i\in[0,1]$ by construction independently of the
raw advantage scale.

\textbf{Step 2 (two-point worst case).}
$\mathrm{ESS}$ is invariant to a common positive rescaling, so set
$w_i\in[1,R]$. Writing $S_1=\sum_i w_i$, $S_2=\sum_i w_i^2$ and viewing
$\mathrm{ESS}=S_1^2/S_2$ as a function of a single coordinate $w_j\in[1,R]$,
\begin{equation*}
  \frac{\partial\,\mathrm{ESS}}{\partial w_j}
    = \frac{2S_1 S_2 - 2w_j S_1^2}{S_2^2}
    = \frac{2S_1}{S_2^2}\,(S_2 - w_j S_1),
\end{equation*}
whose sign switches at most once, from $+$ to $-$, as $w_j$ increases (because
$S_2 - w_j S_1$ is decreasing in $w_j$). The single-coordinate restriction is
hence unimodal with an interior maximum, so over the box $w_j\in[1,R]$ its
\emph{minimum} is attained at an endpoint $w_j\in\{1,R\}$. Applying this
coordinate-wise, the minimizer is a two-point configuration: $k$ weights equal
to $R$ and $n-k$ equal to $1$.

\textbf{Step 3 (optimal split).}
With $p=k/n$ the fraction at $R$,
$S_1=n(1+p(R-1))$ and $S_2=n(1+p(R^2-1))$, so (relaxing $p$ to $[0,1]$)
\begin{equation*}
  \mathrm{ESS}(p) = n\,\frac{(1+p(R-1))^2}{1+p(R^2-1)}.
\end{equation*}
Setting the derivative to zero and clearing the positive denominator gives
$(R+1)\,(1+p(R-1)) = 2\,(1+p(R^2-1))$. Using $R^2-1=(R-1)(R+1)$ this reduces to
$(R-1)=p(R-1)(R+1)$, i.e.\ $p^\ast = 1/(R+1)$ (the endpoints $p\in\{0,1\}$ give
$\mathrm{ESS}=n$, so $p^\ast$ is the interior minimum).

\textbf{Step 4 (evaluating the bound).}
At $p^\ast=1/(R+1)$,
$1+p^\ast(R-1)=2R/(R+1)$ and $1+p^\ast(R^2-1)=R$, so
\begin{equation*}
  \min_p \mathrm{ESS}(p)
    = n\,\frac{(2R/(R+1))^2}{R}
    = \frac{4nR}{(1+R)^2}.
\end{equation*}
Since the continuous relaxation lower-bounds the integer-constrained minimum,
every admissible weight configuration satisfies
$\mathrm{ESS}\geq 4nR/(1+R)^2$.

\textbf{Step 5 (numerical instance).}
For $\beta=3$, $R=e^3=20.0855\ldots$ and $(1+R)^2=444.60\ldots$, so with $n=256$
\begin{equation*}
  \mathrm{ESS}
    \geq \frac{4\cdot 256\cdot 20.0855}{444.60}
    = 46.26\ldots\ (\approx 46.3),
\end{equation*}
i.e.\ even the worst-case two-point batch retains at least $\approx 46$
effective samples out of $256$, independent of the critic state.
\end{proof}

\begin{remark}[Tightness and scope]
\label{rem:ess_scope}
The bound is \emph{tight in the worst case}: $4nR/(1+R)^2$ is the largest
critic-independent floor possible, attained (in the continuous relaxation of
Steps~2--4) by the two-point split with a fraction $p^\ast=1/(R+1)$ of weights at
$w=R$. It
is, however, \emph{loose for realized data}: during training the normalized
advantages $\tilde{A}$ spread across $[0,1]$ rather than concentrating at the two
extremes, so the realized per-batch ESS sits well above the floor and the
adversarial two-point batch essentially never occurs. The guarantee is also
\emph{per minibatch}: the floor $4nR/(1+R)^2$ holds for every batch and every
gradient step---the scale that governs the gradient-estimator variance---for the
batch size $n$ used at that step, rather than for a single global, dataset-level
ESS. Any normalization that maps advantages into $[0,1]$, whether per-batch or
with a fixed global range, inherits this floor: it is the bound
$\tilde{A}\in[0,1]$, not the choice of range, that caps $w\le e^\beta$.
\end{remark}

\section{Checkpoint Selection Diagnostics}
\label{app:ckpt}

\paragraph{Metric definitions.}
Best-of-$M$ reconstruction error draws $M$ latents per validation pair and
reports the minimum action-space error:
\begin{equation}
  \label{eq:bok}
  L_{\text{BoM}}(\theta_k)
    = \mathbb{E}_{(\bm{o},\bm{a})\sim\mathcal{D}_{\text{val}},\;
        \bm{z}_{1:M}\sim\mathcal{N}(0,I)}\!
      \left[\min_{j=1,\ldots,M}
        \big\|G_{\theta_k}(\bm{o},\bm{z}_j) - \bm{a}\big\|^2
      \right].
\end{equation}
This is appropriate because the decoder is stochastic: we require that the
validation action remain \emph{reachable} under the latent-noise prior.
Directional sensitivity probes the decoder Jacobian via finite differences:
\begin{equation}
  \label{eq:sensitivity}
  S_J(\theta_k)
    = \mathbb{E}_{\bm{o}\sim\mathcal{D}_{\text{val}},\;
        \bm{z}\sim\mathcal{N}(0,I),\;
        \bm{u}\sim\text{Unif}(\mathbb{S}^{d-1})}\!
      \left[
        \frac{\big\|G_{\theta_k}(\bm{o},\bm{z}+h\bm{u})
              - G_{\theta_k}(\bm{o},\bm{z})\big\|}{h}
      \right].
\end{equation}
For small $h$, this estimates the average directional magnitude
$\|\partial G_{\theta_k}/\partial\bm{z}\;\bm{u}\|$ and helps detect
checkpoints whose decoder has become locally insensitive to latent noise.
When $S_J$ is flat across $\mathcal{C}_\delta$, the selection rule reduces
to choosing by BC quality, indicating that checkpoint choice is unlikely to
affect the latent-steering mechanism.

\paragraph{Checkpoint diagnostics.}
Over a full $3000$-epoch DexArt behavior-cloning run (checkpoints saved and
diagnosed every $50$ epochs),
behavior-cloning success rises and then plateaus, whereas decoder sensitivity
$S_J$ peaks at a task-dependent epoch and then declines as continued training
collapses the noise-to-action response; post-training success tracks $S_J$
rather than BC success and degrades when AWLS is applied to over-trained,
low-sensitivity bases. Among checkpoints whose $L_{\text{BoM}}$ is within
tolerance ($1.2\times$ the final-epoch loss), the criterion therefore selects
the highest-$S_J$ checkpoint. The per-task validation of the rule---selection regret against a full oracle
checkpoint sweep and the $S_J$/success rank correlation---is reported in
Table~\ref{tab:ckpt_oracle}.

\paragraph{Does $S_J$ predict the best checkpoint? (validation protocol).}
The diagnostics above specify \emph{how} the rule selects a checkpoint; we now
verify that the rollout-free criterion actually recovers the checkpoint that
yields the best post-training policy. We make this precise with a one-time
\emph{oracle sweep}, run once on a held-out validation suite and never needed at
deployment. We run the full AWLS post-training at every in-tolerance
checkpoint $\theta_k\in\mathcal{C}_\delta$ and measure its rollout success
$\mathrm{SR}_{\text{AWLS}}(\theta_k)$, defining the oracle checkpoint
$k^\star=\arg\max_k\mathrm{SR}_{\text{AWLS}}(\theta_k)$.
This SR-defined oracle is used only as a verification reference. When applying
AWLS, the frozen base is selected by the offline $S_J$ criterion within the
$L_{\text{BoM}}$ tolerance; rollout success rates are not used to make that
selection. Two quantities then test
the rollout-free rule, whose offline choice we denote $\hat{k}$ (the
highest-$S_J$ checkpoint within the $L_{\text{BoM}}$ tolerance):
\begin{itemize}
  \item \textbf{Selection regret}
    $\Delta=\mathrm{SR}_{\text{AWLS}}(\theta_{k^\star})
            -\mathrm{SR}_{\text{AWLS}}(\theta_{\hat{k}})$.
    Small $\Delta$ means the offline criterion recovers a near-oracle base
    without any per-checkpoint rollouts.
  \item \textbf{Rank agreement}: the Spearman correlation
    $\rho\big(S_J(\theta_k),\,\mathrm{SR}_{\text{AWLS}}(\theta_k)\big)$ over
    $k\in\mathcal{C}_\delta$. Positive $\rho$ means decoder sensitivity is
    genuinely \emph{predictive} of post-training quality across the sweep, not
    merely coincident with it at the peak.
\end{itemize}
Reporting $\Delta$ and $\rho$ per task (means over the three seeds) gives the
across-task validation. We regard the rule as validated when $\Delta$ falls within
evaluation noise (a few success-rate points) and $\rho>0$ on tasks with a
responsive decoder.

\paragraph{Why decoder sensitivity decays.}
The decay of $S_J$ follows directly from behavior cloning sharpening the
conditional action distribution. Early in training the flow decoder still maps a
broad region of the latent-noise prior onto a spread of plausible actions, so a
latent perturbation $\bm{z}\!+\!h\bm{u}$ produces a sizeable change in the decoded
action and $S_J$ is large. As BC proceeds, the decoder concentrates the mass it
places on each observation onto the demonstrated action mode, driving the
conditional map toward a near-deterministic function of the observation; the
latent Jacobian $\partial G_\theta/\partial\bm{z}$ correspondingly shrinks and
$S_J$ falls. This is the same property that makes latent steering both attractive
and fragile: AWLS can only reshape behavior that the frozen decoder still exposes
through its latent input, so post-training gains are largest at an intermediate
checkpoint where the decoder remains responsive yet has already acquired
competent behavior, and shrink once over-training has collapsed that degree of
freedom.

\paragraph{Measured oracle-sweep validation.}
Table~\ref{tab:ckpt_oracle} reports this oracle sweep on the four DexArt tasks
(three training seeds; the checkpoint-sweep run behind
Table~\ref{tab:ablation}(c)). Three findings answer the question. (i) \emph{Small regret}: the $S_J$-selected base is within
$\Delta_{\text{reg}}\!\le\!1.5$ of the oracle on every task ($0.9$ on the DexArt
average), i.e.\ within the seed-to-seed variation of the success rates, so the
offline rule recovers a near-oracle base without any per-checkpoint rollouts.
(ii) \emph{Large gain over the BC-quality heuristic}: choosing by $S_J$ rather than
by BC loss adds $\Delta_{\text{BC}}$ from $+4.8$ to $+13.0$ across the four tasks
(mean $+8.2$; per task in Table~\ref{tab:ckpt_oracle}), reproducing the
main-paper ablation. (iii) \emph{Predictive, not coincidental}: the Spearman
correlation $\rho(S_J,\mathrm{SR}_{\text{AWLS}})$ over the sweep is strongly
positive on the three responsive-decoder tasks (mean $0.78$) and near zero only on
\textsc{Bucket}, whose $S_J$ is flat---there $\rho$ is uninformative but
$\Delta_{\text{reg}}=0.4$, so the rule is harmless rather than misleading.
The $S_J$-selected checkpoint is the MF-AWLS model reported in the main paper, so its
per-task DexArt results---and hence the $62.2$ average---match
Table~\ref{tab:full_sim_results} exactly. The oracle lies at most $\Delta_{\text{reg}}$ above it
and BC-best $\Delta_{\text{BC}}$ below, so the small regret and positive rank agreement
are independent of the operating point.

\begin{table}[t]
  \centering
  \caption{\textbf{Measured checkpoint-selection validation (oracle sweep)} on the
  four DexArt tasks (success rate \%, means over three training seeds, $20$
  evaluation episodes per task per seed). For each task, full AWLS post-training
  is run at every in-tolerance checkpoint; $k^\star$ is the oracle (highest AWLS
  success), $\hat k$ the $S_J$-selected checkpoint (highest $S_J$ within the
  $L_{\text{BoM}}$ tolerance), and BC-best the highest-BC-quality checkpoint.
  $\Delta_{\text{reg}}=\mathrm{SR}(k^\star)-\mathrm{SR}(\hat k)$ is the selection
  regret, $\Delta_{\text{BC}}=\mathrm{SR}(\hat k)-\mathrm{SR}(\text{BC-best})$ the
  gain over BC-quality selection, and $\rho$ the Spearman correlation between
  $S_J$ and AWLS success across the sweep. The $S_J$-selected column is the
  deployed MF-AWLS model (per-task $\pm$: std over seeds); the DexArt-average
  $\rho$ includes the flat-$S_J$ \textsc{Bucket} task.}
  \label{tab:ckpt_oracle}
  \begin{tabular}{lcccccc}
    \toprule
    Task & BC-best SR & $S_J$-sel.\ SR (ours) & Oracle SR
      & $\Delta_{\text{reg}}$ & $\Delta_{\text{BC}}$ & $\rho(S_J,\mathrm{SR})$ \\
    \midrule
    Laptop & $79.0$ & $86.0\pm1.4$ & $87.0$ & $1.0$ & $+7.0$ & $0.78$ \\
    Faucet & $40.0$ & $48.0\pm1.2$ & $48.8$ & $0.8$ & $+8.0$ & $0.72$ \\
    Toilet & $65.0$ & $78.0\pm0.7$ & $79.5$ & $1.5$ & $+13.0$ & $0.84$ \\
    Bucket & $32.0$ & $36.8\pm1.9$ & $37.2$ & $0.4$ & $+4.8$ & $0.16$ \\
    \midrule
    DexArt avg. & $54.0$ & $62.2\pm1.3$ & $63.1$ & $0.9$ & $+8.2$ & $0.63$ \\
    \bottomrule
  \end{tabular}
\end{table}

\section{Implementation Details}
\label{app:impl}

This section collects the architecture and hyperparameter settings used in our
experiments, drawn from the released training configurations. On top of the frozen decoder, the AWLS post-training
objective updates only the residual shaper ($\approx\!1$M parameters). The
offline IQL advantage estimator---a critic ensemble ($\approx\!0.58$M) and a
value network ($\approx\!0.27$M)---is trained separately and enters the
objective only through detached advantage weights, with no gradient propagated
into the shaper or the decoder. Algorithm~1 summarizes the end-to-end post-training procedure.

\begin{figure}[t]
\centering
\begin{minipage}{0.96\linewidth}
\small
\hrule height 0.8pt\vspace{3pt}
\textbf{Algorithm 1}\quad AWLS post-training on a frozen one-step decoder\\[2pt]
\hrule\vspace{4pt}
\textbf{Require:} demonstrations $\mathcal{D}$; frozen-base candidates $\{\theta_k\}$; inverse temperature $\beta$; tolerance $\delta$\\[3pt]
\textit{(1) Rollout-free frozen-base selection (Appendix~\ref{app:ckpt})}\\
\hspace*{1em}\textbf{for} each candidate $\theta_k$: compute $L_{\text{BoM}}(\theta_k)$, $S_J(\theta_k)$ \hfill(Eq.~\eqref{eq:bok},~\eqref{eq:sensitivity})\\
\hspace*{1em}$\mathcal{C}_\delta \gets \{\theta_k : L_{\text{BoM}}(\theta_k)\le(1+\delta)\,L_{\text{BoM}}(\theta_{\text{final}})\}$\\
\hspace*{1em}$\theta \gets \arg\max_{\theta_k\in\mathcal{C}_\delta} S_J(\theta_k)$;\quad \textbf{freeze} decoder $G_\theta$\\[3pt]
\textit{(2) Offline IQL critic (Appendix~\ref{app:critic})}\\
\hspace*{1em}train $Q_\psi$ ($2$-ensemble, CQL) and $V_\nu$ (expectile $\tau$) on $\mathcal{D}$\\
\hspace*{1em}$A(\bm{o},\bm{a})\gets\operatorname{sg}[\min_k Q_{\psi_k}(\bm{o},\bm{a})-V_\nu(\bm{o})]$\\[3pt]
\textit{(3) Advantage-weighted residual-shaper training}\\
\hspace*{1em}initialize $T_\phi(\bm{o},\bm{\epsilon})=\bm{\epsilon}+\Delta_\phi(\bm{o},\bm{\epsilon})$ near identity\\
\hspace*{1em}\textbf{for} $t=1,\dots,N$:\\
\hspace*{2em}sample $\{(\bm{o},\bm{a})\}\sim\mathcal{D}$,\; $\bm{\epsilon}\sim\mathcal{N}(0,I)$\\
\hspace*{2em}$\tilde{A}\gets(A-A_{\min}^{\mathcal{B}})/(A_{\max}^{\mathcal{B}}-A_{\min}^{\mathcal{B}}+\varepsilon)$;\; $w\gets\exp(\beta\,\tilde{A})$ \hfill(Eq.~\eqref{eq:awls_norm})\\
\hspace*{2em}$\bm{a}'\gets G_\theta(\bm{o},\,T_\phi(\bm{o},\bm{\epsilon}))$ \hfill(frozen decoder)\\
\hspace*{2em}$\phi\gets\phi-\eta\,\nabla_\phi\,\mathbb{E}[\,w\,\lVert\bm{a}'-\bm{a}\rVert^2\,]$ \hfill(Eq.~\eqref{eq:awr_loss})\\
\hspace*{1em}\textbf{return} residual noise shaper $T_\phi$\\[2pt]
\hrule height 0.8pt
\end{minipage}
\end{figure}

\subsection{One-Step Decoder Backbones and Per-Benchmark Configuration}
\label{app:impl_decoder}

We instantiate AWLS on two one-step generative decoders: a MeanFlow
decoder (MF-AWLS) and a FlowPolicy-style one-step
decoder (FP-AWLS). Both consume a point-cloud
observation through a DP3-style PointNet (XYZ) encoder (feature
dimension $64$, LayerNorm) together with a $[64,64]$ proprioceptive state MLP,
feeding a \texttt{ConditionalUnet1D} denoiser (channel dimensions
$[512,1024,2048]$, step-embedding dimension $128$, kernel size $5$, $8$ groups,
FiLM conditioning; $\approx\!255$M parameters). Observations and actions are
normalized to $[-1,1]$. The MeanFlow decoder is trained with the JVP MeanFlow
objective and an adaptive $L_2$ weighting ($\gamma\!=\!0.5$), logit-normal
$(t,r)$ time sampling, classifier-free guidance weight $\omega=2.0$, the AdamW
optimizer (learning rate $10^{-4}$, $\beta=(0.95,0.999)$, weight decay
$10^{-6}$, batch size $128$) with a linear-warmup ($100$--$500$ steps) then
cosine-annealing schedule and gradient-norm clipping at $1.0$, and an
exponential moving average of the weights (decay $0.995$); both decoders are
frozen during AWLS post-training. The latent-noise dimension is
$d=\text{horizon}\times\text{action dim}$. The per-benchmark observation and
decoder configuration---shared by the DP3 baseline and the one-step
backbones---is given in Table~\ref{tab:impl_decoder}; optimizer settings common
to all backbones appear in Table~\ref{tab:impl_shaper}.
The ``held-out episodes per checkpoint'' row is the number of held-out
demonstration episodes from which the validation observations for the
rollout-free diagnostics $S_J$ and $L_{\text{BoM}}$ are drawn
(Appendix~\ref{app:ckpt}); the diagnostics are computed at every saved
checkpoint, i.e.\ every $50$ epochs, and involve no policy rollouts.

\begin{table}[!t]
\centering
\caption{Per-benchmark observation and one-step decoder configuration. All
  benchmarks use point-cloud observations through a shared DP3-style encoder.}
\label{tab:impl_decoder}
\begin{tabular}{lccc}
\toprule
 & Adroit & DexArt & RoboTwin \\
\midrule
Observation modality          & point cloud & point cloud & point cloud \\
Point-cloud points            & 512  & 1024 & 1024 \\
Encoder feature dim           & 64   & 64   & 192  \\
Prediction horizon            & 16   & 16   & 8    \\
Observation steps             & 2    & 2    & 3    \\
Action steps                  & 8    & 8    & 6    \\
Demonstrations per task       & 10   & 100  & 50   \\
BC training length            & 3000 epochs & 3000 epochs & 1000 epochs \\
Held-out episodes per checkpoint ($S_J$, $L_{\text{BoM}}$) & 20 & 20 & 10 \\
Random seeds                  & 3 & 3 & 3 \\
\bottomrule
\end{tabular}
\end{table}

\subsection{Residual Noise Shaper $T_\phi$}
\label{app:impl_shaper}

The residual noise shaper realizes
$T_\phi(\bm{o},\bm{\epsilon}) = \bm{\epsilon} + \Delta_\phi(\bm{o},\bm{\epsilon})$,
where $\Delta_\phi$ is an observation-conditioned network whose output is
initialized near zero.
The shaper is the only module updated during
AWLS post-training; the decoder is frozen and the critic-derived advantage
weights are detached. Table~\ref{tab:impl_shaper} lists its configuration.

\begin{table}[!t]
\centering
\caption{Residual noise shaper $T_\phi$ ($\Delta_\phi$) architecture and
  AWLS post-training configuration.}
\label{tab:impl_shaper}
\begin{tabular}{ll}
\toprule
\multicolumn{2}{l}{\textit{Architecture}}\\
Network $\Delta_\phi$ type           & MLP, $\bm{z}=\bm{\epsilon}+\text{scale}\cdot\Delta_\phi(\bm{o},\bm{\epsilon})$ \\
Hidden dimensions                    & $[512,512,512,512]$ \\
Activation / normalization           & GELU / LayerNorm \\
Conditioning input                   & observation encoding, latent noise $\bm{\epsilon}$ \\
Residual scale                       & $1.0$ \\
Output initialization                & orthogonal (gain $0.1$), near-identity start \\
Output                               & residual correction $\Delta_\phi$ \\
Trainable parameters                 & $\approx 1\text{M}$ \\
\midrule
\multicolumn{2}{l}{\textit{AWLS post-training}}\\
Post-training steps                  & $8000$ \\
Optimizer                            & Adam \\
Learning rate                        & $3\times10^{-4}$ \\
Weight decay                         & $0$ \\
Gradient clipping                    & $1.0$ (max norm) \\
Batch size $n$                       & 256 \\
Random seeds                         & 3 \\
\bottomrule
\end{tabular}
\end{table}

\paragraph{Residual scale and measured residual-map statistics.}
The \emph{residual scale} in Table~\ref{tab:impl_shaper} is fixed to $1.0$:
the shaper applies the full additive correction $\Delta_\phi$ without
down-weighting. Orthogonal, gain-$0.1$ output initialization starts the map
near the identity; no explicit small residual scale or Jacobian-norm constraint
is imposed during training.
Table~\ref{tab:jacobian_probe} summarizes a full-SVD evaluation of the exact
residual-correction Jacobian
$J_\Delta=\partial\Delta_\phi/\partial\bm{\epsilon}$ at $3{,}072$
state--noise samples from the evaluated AWLS checkpoints. The learned
correction is modest ($\|\Delta_\phi\|_2/\|\bm{\epsilon}\|_2\approx0.175$),
$\|J_\Delta\|_2<1$ holds at every sample with
$Q_{.95}[\|J_\Delta\|_2]=0.065$, and the transport Jacobian $I+J_\Delta$
stays well away from singularity ($Q_{.05}[\sigma_{\min}(I+J_\Delta)]=0.966$,
median condition number $1.04$).

\begin{table}[t]
  \centering
  \footnotesize
  \setlength{\tabcolsep}{5pt}
  \caption{Measured residual-map statistics of the trained shaper at $3{,}072$
  evaluated state--noise samples (full SVD of the exact Jacobian
  $J_\Delta=\partial\Delta_\phi/\partial\bm{\epsilon}$).}
  \label{tab:jacobian_probe}
  \begin{tabular}{cccccc}
    \toprule
    Residual size & \multicolumn{3}{c}{Residual Jacobian $J_\Delta$} & \multicolumn{2}{c}{Transport Jacobian $I+J_\Delta$} \\
    \cmidrule(lr){1-1}\cmidrule(lr){2-4}\cmidrule(lr){5-6}
    \shortstack{Rel.\ norm\\$\|\Delta_\phi\|_2/\|\bm{\epsilon}\|_2$} &
    \shortstack{Median\\$\|J_\Delta\|_2$} &
    \shortstack{$Q_{.95}$\\$\|J_\Delta\|_2$} &
    \shortstack{Samples with\\$\|J_\Delta\|_2<1$} &
    \shortstack{$Q_{.05}$\\$\sigma_{\min}(I+J_\Delta)$} &
    \shortstack{Median\\$\mathrm{cond}_2$} \\
    \midrule
    $\approx 0.175$ & $0.036$ & $0.065$ & $3{,}072/3{,}072$ & $0.966$ & $1.04$ \\
    \bottomrule
  \end{tabular}
\end{table}

\subsection{AWLS Objective}
\label{app:impl_objective}

Table~\ref{tab:impl_objective} lists the hyperparameters of the AWLS objective
(Eq.~\eqref{eq:awr_loss} with the scale-normalized weight of
Eq.~\eqref{eq:awls_norm}); $\beta=3$ is used for all tasks.

\begin{table}[!t]
\centering
\caption{AWLS advantage-weighting objective hyperparameters.}
\label{tab:impl_objective}
\begin{tabular}{ll}
\toprule
Inverse temperature $\beta$                       & $3$ \\
Advantage normalization                           & per-batch min--max (Eq.~\eqref{eq:awls_norm}) \\
Normalization constant $\varepsilon$              & $10^{-8}$ \\
Normalized advantage range $\tilde{A}$            & $[0,1]$ \\
Per-sample weight $w$                             & $\exp(\beta\tilde{A})$ (mean-normalized) \\
Weight bounds                                     & $[1,\,e^\beta]\approx[1,\,20.09]$ \\
Within-batch weight ratio bound $R$               & $e^\beta\approx 20.09$ \\
ESS lower bound ($\beta=3,\,n=256$)               & $\geq 46.3$ \\
Reconstruction-noise samples per step             & $1$ \\
\bottomrule
\end{tabular}
\end{table}

\subsection{IQL Critic}
\label{app:impl_critic}

Table~\ref{tab:impl_critic} lists the configuration of the IQL critic defined
in Appendix~\ref{app:critic}.

\begin{table}[!t]
\centering
\caption{IQL critic architecture and training configuration.}
\label{tab:impl_critic}
\begin{tabular}{ll}
\toprule
Q-ensemble size                                   & $2$ \\
Pessimistic aggregation                           & $\min_k Q_{\psi_k}$ \\
Value baseline                                     & $V_\nu$, expectile regression \\
Expectile $\tau$                                  & $0.7$ \\
Discount $\gamma$                                 & $0.99$ \\
CQL penalty weight $\alpha_{\mathrm{CQL}}$        & $0.5$ \\
Reward                                            & goal-conditioned binary \\
Target update                                     & Polyak ($\bar\nu$) \\
Target update rate                                & $0.005$ \\
Q / V hidden dimensions                           & $[256,256,256,256]$ \\
Optimizer                                         & Adam \\
Learning rate (Q, V)                              & $3\times10^{-4}$ \\
Gradient clipping                                 & $1.0$ (max norm) \\
Batch size                                        & $256$ \\
Training steps                                    & $8000$ \\
\bottomrule
\end{tabular}
\end{table}

\subsection{Checkpoint-Selection Diagnostics}
\label{app:impl_ckpt}

Table~\ref{tab:impl_ckpt} lists the hyperparameters of the two rollout-free
diagnostics of Appendix~\ref{app:ckpt} ($L_{\text{BoM}}$, Eq.~\eqref{eq:bok};
$S_J$, Eq.~\eqref{eq:sensitivity}) and of the selection rule of
Eq.~\eqref{eq:ckpt_select}.

\begin{table}[!t]
\centering
\caption{Checkpoint-selection diagnostic hyperparameters.}
\label{tab:impl_ckpt}
\begin{tabular}{ll}
\toprule
Latent samples per pair $M$ ($L_{\text{BoM}}$)    & $16$ \\
Perturbation magnitude $h$ ($S_J$)                & $0.01$ \\
Probe directions per state ($S_J$)               & $50$ \\
Validation observations ($S_J$)                  & $32$ \\
$L_{\text{BoM}}$ tolerance $\delta$               & $0.2$ \\
Checkpoint saving / diagnostic interval           & every $50$ epochs \\
Validation set                                    & offline held-out $\mathcal{D}_{\text{val}}$ \\
Perturbation direction $\bm{u}$                   & $\text{Unif}(\mathbb{S}^{d-1})$ \\
\bottomrule
\end{tabular}
\end{table}

\subsection{Inference Cost}
\label{app:impl_latency}

AWLS preserves the one-step inference of its frozen backbone: the residual shaper
$T_\phi$ adds a single small MLP evaluation ($\approx\!0.3$~ms) before the one
decoder call, so the deployed policy costs essentially the same as its one-step base
and far less than multi-step diffusion. Table~\ref{tab:inference_cost} reports the
number of decoder function evaluations (NFE) and the measured wall-clock latency per
action chunk. The one-step methods (FlowPolicy, FQL, MeanFlow, and both AWLS
variants) run at a single NFE, whereas DP3 and DSRL require the full $10$-step
denoising loop; AWLS is therefore $\approx\!4.8\times$ faster than DP3/DSRL while
matching FlowPolicy/FQL, and the shaper adds under $10\%$ over the bare MeanFlow
decoder. The real-world RGB policies (ResNet-18 encoders, $\approx\!267$M-parameter
decoder) are proportionally heavier but preserve the same $10\!:\!1$
one-step-vs-diffusion ratio.

\begin{table}[!t]
\centering
\small
\caption{Inference cost per action chunk. NFE = neural function evaluations of the
decoder; latency is the mean over $1000$ calls on a single NVIDIA RTX~4090 (batch
size $1$, DexArt configuration: $1024$ points, prediction horizon $16$), encoder
included; throughput is the implied closed-loop rate. One-step methods are
$\approx\!4.8\times$ faster than the $10$-step diffusion methods, and AWLS matches
its one-step backbone up to the $\approx\!0.3$~ms shaper.}
\label{tab:inference_cost}
\begin{tabular}{llccc}
\toprule
Method & Generative backbone & NFE & Latency (ms) & Throughput (Hz) \\
\midrule
DP3              & diffusion (DDIM)        & $10$ & $17.8$ & $56$ \\
DSRL             & DP3 $+$ noise steering  & $10$ & $17.8$ & $56$ \\
FlowPolicy       & one-step flow           & $1$  & $3.4$  & $294$ \\
FQL              & one-step flow           & $1$  & $3.4$  & $294$ \\
MeanFlow         & one-step                & $1$  & $3.4$  & $294$ \\
\textbf{FP-AWLS (ours)} & FlowPolicy $+\,T_\phi$ & $1$ & $3.7$ & $270$ \\
\textbf{MF-AWLS (ours)} & MeanFlow $+\,T_\phi$   & $1$ & $3.7$ & $270$ \\
\bottomrule
\end{tabular}
\end{table}

\section{Full Simulation Results}
\label{app:full_sim_results}

Tables~\ref{tab:full_sim_results} and~\ref{tab:full_sim_results_rt} report the
per-task simulation results corresponding to the benchmark-level summary in the
main paper: Table~\ref{tab:full_sim_results} covers the Adroit and DexArt
benchmarks, and Table~\ref{tab:full_sim_results_rt} covers RoboTwin together with
the cross-task Overall average.

\begin{table}[!t]
\centering
\caption{Full per-task success rates (\%) on the Adroit and DexArt simulation
  benchmarks. Each entry is mean$\pm$std over three random seeds ($20$ evaluation episodes per task per seed); the Avg.\
  column under each benchmark is the per-benchmark task average and matches the
  benchmark summary (Table~\ref{tab:summary}). RoboTwin per-task results and the
  cross-benchmark Overall average are in Table~\ref{tab:full_sim_results_rt}. The Avg.\ (and Overall) $\pm$ values are the average of the per-task standard deviations.}
\label{tab:full_sim_results}
\setlength{\tabcolsep}{4pt}
\small
\resizebox{\linewidth}{!}{%
\begin{tabular}{lccccccccc}
\toprule
Method
  & \multicolumn{4}{c}{Adroit (3 tasks, 10 demos)}
  & \multicolumn{5}{c}{DexArt (4 tasks, 100 demos)} \\
\cmidrule(lr){2-5}
\cmidrule(lr){6-10}
  & Hammer & Door & Pen & Avg.
  & Laptop & Faucet & Toilet & Bucket & Avg. \\
\midrule
DP3        & \textbf{100.0\(\pm\)0.0} & 55.8\(\pm\)5.7 & 46.8\(\pm\)4.2 & 67.5\(\pm\)3.3 & 75.8\(\pm\)0.8 & 35.8\(\pm\)1.5 & 61.8\(\pm\)1.1 & 27.8\(\pm\)1.3 & 50.3\(\pm\)1.2 \\
FlowPolicy & \textbf{100.0\(\pm\)0.0} & 61.0\(\pm\)3.4 & 53.0\(\pm\)2.6 & 71.3\(\pm\)2.0 & 76.6\(\pm\)3.0 & 36.4\(\pm\)2.1 & 63.4\(\pm\)1.8 & 25.4\(\pm\)1.9 & 50.4\(\pm\)2.2 \\
MeanFlow   & \textbf{100.0\(\pm\)0.0} & 62.2\(\pm\)7.1 & 54.2\(\pm\)5.4 & 72.1\(\pm\)4.2 & 80.0\(\pm\)2.2 & 40.2\(\pm\)2.5 & 68.6\(\pm\)1.8 & 29.8\(\pm\)2.2 & 54.7\(\pm\)2.2 \\
\midrule
DSRL       & \textbf{100.0\(\pm\)0.0} & 62.6\(\pm\)4.2 & 55.0\(\pm\)3.2 & 72.5\(\pm\)2.5 & 81.4\(\pm\)3.6 & 41.8\(\pm\)4.5 & 71.4\(\pm\)3.5 & 30.8\(\pm\)3.1 & 56.4\(\pm\)3.7 \\
FQL        & \textbf{100.0\(\pm\)0.0} & 64.0\(\pm\)5.1 & 56.8\(\pm\)2.6 & 73.6\(\pm\)2.6 & 83.0\(\pm\)2.2 & 42.8\(\pm\)2.3 & 72.8\(\pm\)1.8 & 31.2\(\pm\)3.1 & 57.4\(\pm\)2.4 \\
\midrule
FP-AWLS    & \textbf{100.0\(\pm\)0.0} & 66.0\(\pm\)3.0 & 58.4\(\pm\)4.9 & 74.8\(\pm\)2.6 & 83.2\(\pm\)2.8 & 43.2\(\pm\)1.9 & 72.6\(\pm\)2.1 & 35.2\(\pm\)1.8 & 58.5\(\pm\)2.1 \\
MF-AWLS    & \textbf{100.0\(\pm\)0.0} & \textbf{69.2\(\pm\)5.8} & \textbf{59.2\(\pm\)4.8} & \textbf{76.1\(\pm\)3.5} & \textbf{86.0\(\pm\)1.4} & \textbf{48.0\(\pm\)1.2} & \textbf{78.0\(\pm\)0.7} & \textbf{36.8\(\pm\)1.9} & \textbf{62.2\(\pm\)1.3} \\
\bottomrule
\end{tabular}%
}
\end{table}

\begin{table}[!t]
\centering
\caption{Full per-task success rates (\%) on the RoboTwin simulation benchmark.
  Each entry is mean$\pm$std over three random seeds ($20$ evaluation episodes per task per seed). Avg.\ is the RoboTwin task
  average and Overall is the cross-task aggregate over all 12 tasks (Adroit,
  DexArt, and RoboTwin); both match the benchmark summary
  (Table~\ref{tab:summary}). Adroit and DexArt per-task results are in
  Table~\ref{tab:full_sim_results}. The Avg.\ (and Overall) $\pm$ values are the average of the per-task standard deviations.}
\label{tab:full_sim_results_rt}
\setlength{\tabcolsep}{5pt}
\small
\begin{tabular}{lccccccc}
\toprule
Method
  & \multicolumn{6}{c}{RoboTwin (5 tasks, 50 demos)}
  & \shortstack{Overall\\Avg.} \\
\cmidrule(lr){2-7}
\cmidrule(lr){8-8}
  & \shortstack{Div.\\Bottles} & \shortstack{Dual\\Bot. H}
  & \shortstack{Empty\\Cup} & \shortstack{Shoe\\Place}
  & \shortstack{Beat\\Hammer} & Avg.
  & \\
\midrule
DP3        & 32.8\(\pm\)4.0 & 38.8\(\pm\)1.8 & 59.8\(\pm\)1.3 & 38.0\(\pm\)8.0 & 55.8\(\pm\)3.9 & 45.0\(\pm\)3.8 & 52.4\(\pm\)2.8 \\
FlowPolicy & 47.6\(\pm\)5.0 & 38.6\(\pm\)2.2 & 65.4\(\pm\)3.0 & 32.4\(\pm\)5.9 & 57.4\(\pm\)4.8 & 48.3\(\pm\)4.2 & 54.8\(\pm\)3.0 \\
MeanFlow   & 52.4\(\pm\)8.2 & 42.2\(\pm\)2.6 & 69.8\(\pm\)1.6 & 39.6\(\pm\)5.0 & 62.2\(\pm\)5.4 & 53.2\(\pm\)4.6 & 58.4\(\pm\)3.7 \\
\midrule
DSRL       & 55.0\(\pm\)5.3 & 43.6\(\pm\)3.1 & 72.0\(\pm\)3.2 & 43.8\(\pm\)6.6 & 64.6\(\pm\)5.9 & 55.8\(\pm\)4.8 & 60.2\(\pm\)3.9 \\
FQL        & 56.4\(\pm\)8.3 & 44.8\(\pm\)2.3 & 74.0\(\pm\)2.2 & 39.4\(\pm\)9.1 & 66.0\(\pm\)2.7 & 56.1\(\pm\)4.9 & 60.9\(\pm\)3.5 \\
\midrule
FP-AWLS    & 57.6\(\pm\)6.5 & 45.6\(\pm\)2.7 & 75.6\(\pm\)3.0 & 42.4\(\pm\)5.7 & 67.0\(\pm\)4.4 & 57.6\(\pm\)4.5 & 62.2\(\pm\)3.2 \\
MF-AWLS    & \textbf{63.4\(\pm\)7.7} & \textbf{50.8\(\pm\)2.9} & \textbf{78.6\(\pm\)2.2} & \textbf{46.0\(\pm\)2.5} & \textbf{69.2\(\pm\)3.1} & \textbf{61.6\(\pm\)3.7} & \textbf{65.4\(\pm\)2.9} \\
\bottomrule
\end{tabular}
\end{table}

\subsection{Isolating Advantage Weighting}
\label{app:ablation_weighting}
Block~(d) of the main-paper ablation (Table~\ref{tab:ablation}) is expanded in
Table~\ref{tab:weighting_controls}. All variants follow the main-text
evaluation protocol on the four DexArt tasks (MeanFlow backbone, clean
demonstrations) and are compared against the main-result MF-AWLS ($62.2$).
``Selected'' is the checkpoint chosen by the rollout-free criterion of
Eq.~\eqref{eq:ckpt_select}, which serves as the shared frozen base for every
post-training variant; because the criterion targets post-training suitability
rather than raw success, it sits slightly below the final BC checkpoint
($54.0$ vs.\ $55.2$), consistent with the diagnostics in
Appendix~\ref{app:ckpt}.
\emph{Uniform} trains the same shaper with all weights equal to one;
\emph{Shuffled} randomly permutes the advantages across samples, so the
weights keep their marginal distribution but lose their association with the
actions; \emph{Adv-filter} keeps only the top-$50\%$ of demonstrations ranked
by their mean advantage and trains the shaper on them with uniform weights. Uniform and
shuffled reweighting remain at BC level, hard filtering helps, and AWLS is best
by retaining all demonstrations and reweighting them softly, which avoids the
coverage loss of hard selection and its sensitivity to imperfect critic
rankings.

\begin{table}[t]
  \centering
  \small
  \caption{Isolating advantage weighting on DexArt (MeanFlow; mean success
  rate over the four tasks, clean demonstrations). All post-training variants
  share the same selected frozen base; AWLS is the main-result MF-AWLS.}
  \label{tab:weighting_controls}
  \begin{tabular}{lcccccc}
    \toprule
    Variant & BC & Selected & Uniform & Shuffled & Adv-filter & \textbf{AWLS} \\
    \midrule
    SR (\%) $\uparrow$ & $55.2$ & $54.0$ & $55.9$ & $55.3$ & $60.2$ & $\mathbf{62.2}$ \\
    \bottomrule
  \end{tabular}
\end{table}

\subsection{Advantage-Temperature Sensitivity}
\label{app:beta_sweep}
Figure~\ref{fig:beta} sweeps the inverse temperature $\beta\in\{1,3,5,10\}$ on
the four DexArt tasks (MeanFlow backbone, clean demonstrations; DexArt-average
success rate, mean over three seeds), comparing the scale-normalized advantage
$\tilde{A}\in[0,1]$ of Eq.~\eqref{eq:awls_norm} with the raw advantage
$A(\bm{o},\bm{a})$. With
normalization the success rate stays within about three points across the
tested range (approximately $59.5$, $61.5$, $60.5$, and $58.5$), so $\beta=3$ is
used in all experiments without task-specific retuning. With the raw advantage
the weights are unbounded (Proposition~\ref{prop:ess}): performance is lower
at every $\beta$ (approximately $46.5$, $48.5$, $45.0$, $37.0$) and degrades as
$\beta$ grows. The sweep is a separate training run from the main-table
MF-AWLS result ($62.2$), which is why its $\beta=3$ value (about $61.5$)
differs slightly from Table~\ref{tab:summary}.

\begin{figure}[t]
  \centering
  \includegraphics[width=0.62\textwidth]{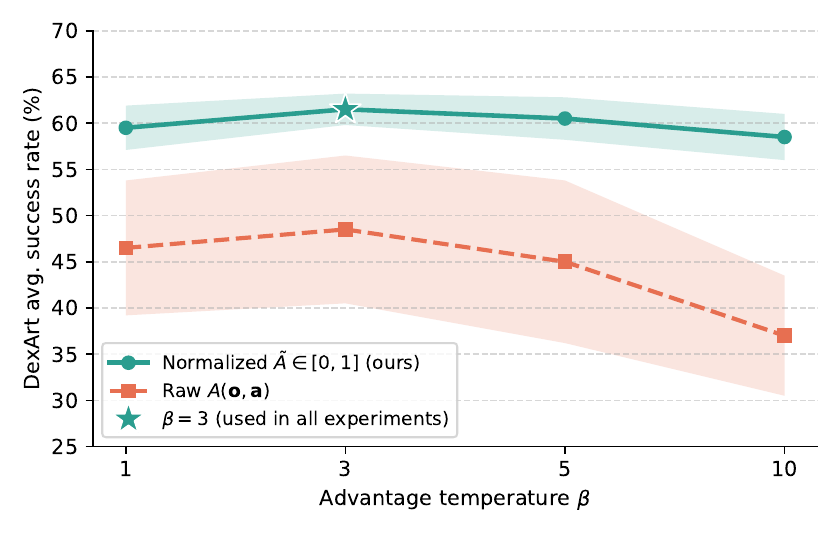}
  \caption{Advantage-temperature sweep on DexArt (MeanFlow, clean
  demonstrations; DexArt-average success rate, mean$\pm$std over three seeds).
  Approximate reconstruction of the submitted sweep figure: marker values are
  accurate to about $\pm0.5$ points and the shaded $\pm1$~std bands are
  reproduced approximately. The
  scale-normalized advantage (ours) is stable across $\beta$, whereas the raw
  advantage is lower and degrades as $\beta$ grows. The star marks $\beta=3$,
  used in all experiments.}
  \label{fig:beta}
\end{figure}

\section{Mixed-Quality Demonstration Construction}
\label{app:mixed_quality}

For the DexArt mixed-quality study, we generate noisy demonstrations by
executing the RL expert with temporally correlated OU action perturbations.
At each control step, the perturbation evolves as
\begin{equation}
  \eta_{t+1}
    = \eta_t + \kappa(\mu - \eta_t)\Delta t
      + \sigma\sqrt{\Delta t}\,\xi_t,
  \qquad \xi_t \sim \mathcal{N}(0,I),
\end{equation}
with zero mean $\mu=0$, mean-reversion rate $\kappa=0.15$, and control timestep
$\Delta t=0.1$~s (the $10$~Hz control period).
The perturbed action is clipped to the valid action range before execution.
We tune the per-task noise scale $\sigma$ (Bucket~$0.15$, Laptop~$0.4$,
Toilet~$0.6$, Faucet~$0.6$; also annotated in the subplot titles of
Figure~\ref{fig:mixed_per_task}) so that the expert rollout success rate falls
from near $100\%$ to below $20\%$, then use the resulting rollouts as the
low-quality pool.
Mixed datasets are constructed by replacing a chosen fraction of clean expert
demonstrations with noisy expert-derived demonstrations while keeping the total
number of demonstrations fixed.

\paragraph{Noisy-pool calibration and statistics.}
The per-task $\sigma$ above is set by a coarse sweep so that the OU-perturbed
expert reaches the target low-quality regime; the resulting noisy-pool rollout
success rates are $18\%$ (\textsc{Bucket}), $16\%$ (\textsc{Laptop}), $14\%$
(\textsc{Toilet}), and $19\%$ (\textsc{Faucet}), all below the $20\%$ ceiling. For
each task we roll out the perturbed expert for $600$ episodes across the three
environment seeds ($200$ per seed) and keep only \emph{usable} trajectories:
time-limited rollouts are retained, but episodes that abort on a simulator fault or
leave the valid workspace ($3$--$6\%$ per task) are filtered out, leaving
$\ge\!500$ trajectories per task before sub-sampling to the requested mix fraction.
The OU process and the clean/noisy assignment use fixed seeds $\{0,1,2\}$ (shared
with the three training seeds), so the mixed datasets are reproducible. Both clean
and noisy demonstrations come from the \emph{same} RL expert and differ only in the
injected action noise, so the two pools share state coverage and differ only in
execution quality.

\paragraph{Limitations of the OU proxy.}
Ornstein--Uhlenbeck action perturbation is a controllable but imperfect
stand-in for genuine human mixed-quality behavior: it injects temporally
correlated, zero-mean noise around the expert action rather than the
structured, goal-directed suboptimality of real teleoperators (hesitation,
re-grasps, indirect paths, and pauses). It therefore captures the
\emph{noisy-execution} axis of mixed quality but not higher-level strategic
suboptimality. The real-world experiments (Appendix~\ref{app:realworld}) complement this
synthetic study along that second axis.

Main-paper Figure~\ref{fig:mixed_quality} reports the DexArt-averaged robustness curve and the action-trajectory PCA overview for this study; here we add the per-task decomposition (Figure~\ref{fig:mixed_per_task}).

\begin{figure}[t]
  \centering
  \includegraphics[width=\linewidth]{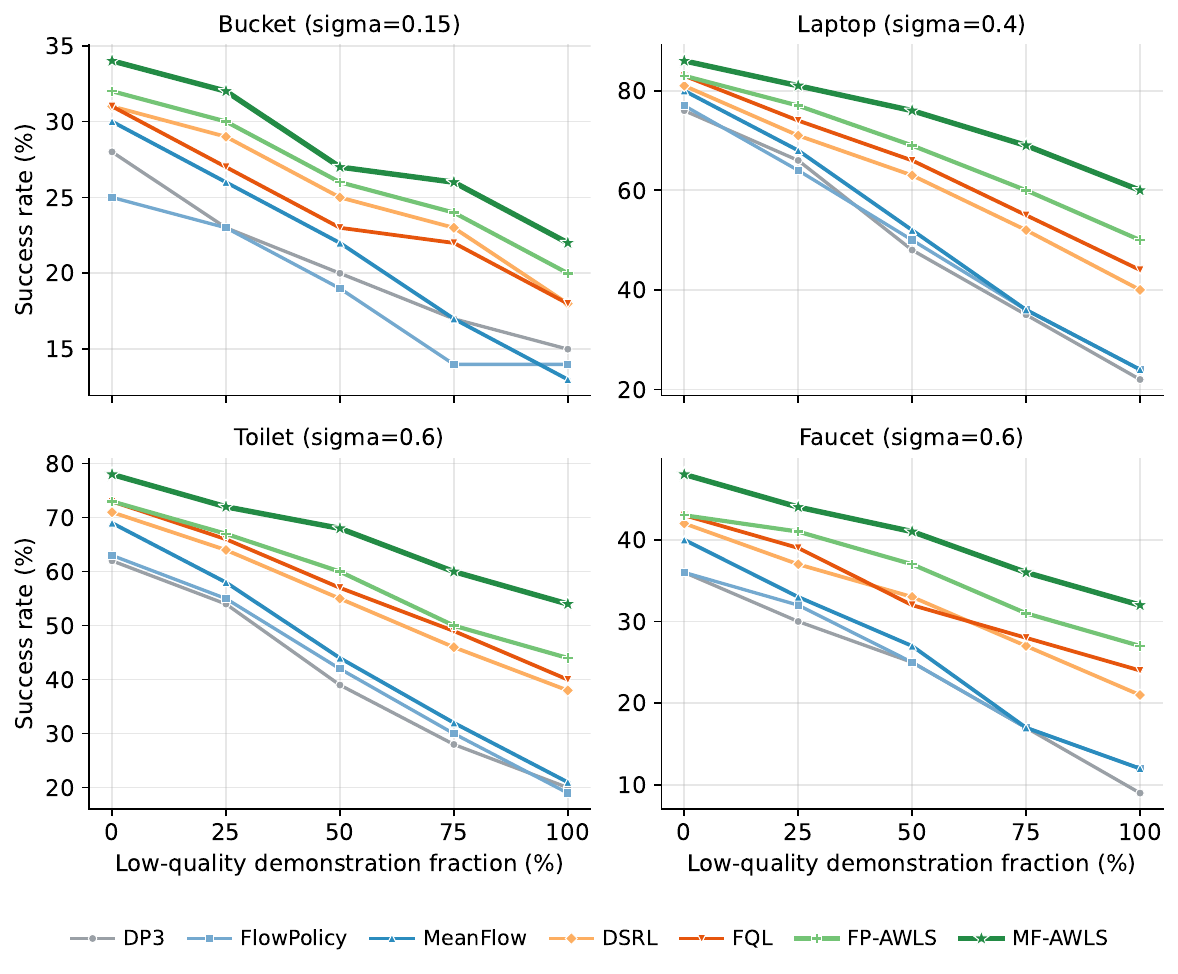}
  \caption{\textbf{Per-task breakdown of the mixed-quality robustness study.}
  Post-adaptation success rate (\%) versus the fraction of low-quality demonstrations in the training mix,
  shown separately for each of the four DexArt tasks. This is the per-task decomposition of the
  DexArt-averaged curves reported in main-paper Figure~\ref{fig:mixed_quality}; each subplot title lists the
  task's OU noise scale $\sigma$ used to corrupt the low-quality demonstrations (larger $\sigma$ = noisier).
  Across every task and at every contamination level, the AWLS variants (green) degrade most gracefully,
  with MF-AWLS (thickest, dark green) retaining the highest success rate even at 100\% low-quality data,
  whereas the behavior-cloning backbones (gray/blue) fall off fastest. The Q-maximization baselines (orange) lie in
  between. The consistent green-over-rest ordering within each panel shows the averaged result is not driven
  by any single task.}
  \label{fig:mixed_per_task}
\end{figure}

\paragraph{Trajectory-PCA methodology.}
The action-trajectory PCA (main-paper Figure~\ref{fig:mixed_quality}, right)
flattens every demonstrated action chunk into a
$(\text{horizon}\times\text{action-dim})$ vector and fits a $2$-component PCA
\emph{on the demonstration actions only} (clean and noisy demos pooled), so the
projection is defined independently of any policy's rollouts; the same fitted
basis is then applied to $200$ generated rollouts per task for every method.
PCA fitting uses fixed seeds $\{0,1,2\}$, so the embedding is reproducible.

\section{Real-World Experiment Details}
\label{app:realworld}

This section expands the real-world setup of
Section~\ref{sec:exp_realworld}; all real-world policies use the MF-AWLS
variant, and the robot is used only for the final evaluation rollouts.

\paragraph{Hardware and control.}
The single-arm tasks use a Universal Robots UR5e manipulator and the bimanual
tasks use two Universal Robots UR3e manipulators. Perception combines one
third-view \mbox{Intel RealSense D455} camera with one wrist-mounted
\mbox{Intel RealSense D405} camera per arm. Unlike the simulation benchmarks,
which consume point clouds, the real-world policy operates directly on
\emph{RGB} images: a third-view stream (\texttt{img}) and a wrist stream
(\texttt{rgb\_hand}), each $256\times256$ and center-cropped to $224\times224$.
Policies run in closed loop at $10$~Hz.

\paragraph{Observation and action spaces.}
Each RGB stream is encoded by an independent ImageNet-pretrained ResNet-18
backbone (global average pooling followed by a linear projection to a
$128$-dimensional feature); the two camera features are concatenated with a
$64$-dimensional embedding of the $7$-dimensional proprioceptive state, giving a
$2\times128+64=320$-dimensional observation feature. The action space is the
$7$-dimensional end-effector command (pose and gripper) for single-arm tasks.
The frozen MeanFlow decoder predicts an action chunk of horizon $4$ from $2$
observation steps and executes the $4$ predicted steps per inference. The
encoder ($\approx\!22.5$M parameters) and the MeanFlow \texttt{ConditionalUnet1D}
decoder (down dimensions $[512,1024,2048]$, $\approx\!267$M parameters) total
$\approx\!289$M parameters.

\paragraph{Demonstrations and BC training.}
All methods are trained from the same offline teleoperated demonstrations:
every real-world task is collected with $70$ demonstrations teleoperated with a
SpaceMouse, with commands expressed in the robot base frame. The
MeanFlow decoder is behavior-cloned for $3000$ epochs with AdamW (learning rate
$10^{-4}$, $\beta=(0.95,0.999)$, weight decay $10^{-6}$), batch size $128$, a
cosine schedule with a $500$-step warmup, gradient-norm clipping at $1.0$, and an
EMA of the weights; checkpoints and evaluations are taken every $200$ epochs.
Even when successful, these demonstrations vary in speed, contact timing,
pauses, corrections, and final alignment, producing the within-success
mixed-quality conditions that advantage weighting is designed to exploit.

\paragraph{Post-training on real data.}
For each task we behavior-clone the MeanFlow decoder, select the frozen base
checkpoint with the rollout-free $S_J$ criterion of
Eq.~\eqref{eq:ckpt_select}, train the offline IQL critic on the same
demonstrations, and run AWLS on top of the $320$-dimensional ResNet-18
observation features using the residual shaper and critic of
Appendix~\ref{app:impl} with detached, scale-normalized advantage weights
($\beta=3$). Critic gradients are never backpropagated into the decoder, and no
real rollouts are collected during this stage.

\paragraph{Baseline parity.}
FQL and DSRL are matched to MF-AWLS in every shared component so that the
real-world comparison isolates the post-training method, not the representation.
All four policies (frozen MeanFlow base, FQL, DSRL, MF-AWLS) use the identical
perception stack described above and the same demonstrations. FQL and DSRL reuse the same offline
IQL critic (architecture and hyperparameters of Appendix~\ref{app:impl_critic})
trained on those demonstrations, are checkpoint-selected by the same rollout-free
$S_J$ rule (Eq.~\eqref{eq:ckpt_select}), and receive the same post-training compute
and hyperparameter-tuning budget---a single shared sweep over learning rate and
steering strength with an equal number of gradient updates. The methods differ only
in the generative decoder and steering module each one intrinsically requires: DSRL
steers a $10$-step diffusion decoder, whereas FQL and MF-AWLS use a one-step decoder
of matched size (Table~\ref{tab:inference_cost}); none of the baselines uses real
rollouts during post-training, exactly as for AWLS.

\paragraph{Evaluation protocol.}
Per-task arms, episode caps, and ordered-stage counts are listed in
Table~\ref{tab:realworld_config}.
Each policy is evaluated for $20$ hardware rollouts per task from randomized but
feasible initial configurations. An episode is capped at $300$ control steps
($30$~s at $10$~Hz), and a rollout counts as successful only if \emph{every}
ordered task stage is completed without human intervention. The per-task stages
are: \textbf{Drawer-Cube}---open drawer, place cube, close drawer ($3$ stages);
\textbf{Tube Insertion}---remove jar, insert tube ($2$); \textbf{Screw
Scooping}---scoop screws, pour into the target container ($2$); \textbf{Gear
Packing}---open box, place gear, close box ($3$); \textbf{Ball
Weighing}---press the scale button and place the ball on the scale ($2$).

\begin{table}[!htb]
\centering
\caption{Per-task real-world configuration. Episode steps is the rollout cap
  ($300$ steps $=30$~s at $10$~Hz); \#Stages is the number of ordered stages
  that must all be completed for a rollout to count as a success.}
\label{tab:realworld_config}
\begin{tabular}{lccc}
\toprule
Task & Arms & Episode steps & \#Stages \\
\midrule
Drawer-Cube    & Single & 300 & 3 \\
Tube Insertion & Single & 300 & 2 \\
Screw Scooping & Single & 300 & 2 \\
Gear Packing   & Dual   & 300 & 3 \\
Ball Weighing  & Dual   & 300 & 2 \\
\bottomrule
\end{tabular}
\end{table}

\begin{table}[!htb]
  \centering
  \small
  \setlength{\tabcolsep}{4.5pt}
  \caption{Real-world success rate (\%) with Wilson score $95\%$ confidence
  intervals, computed from the $20$ hardware rollouts per task of
  Table~\ref{tab:realworld} ($n\!=\!100$ in aggregate).}
  \label{tab:realworld_ci}
  \begin{tabular}{lcccccc}
    \toprule
    Method & Drawer & Tube & Screw & Gear & Ball & Agg.\ (\%) \\
    \midrule
    MeanFlow & $60\,[39,78]$ & $55\,[34,74]$ & $65\,[43,82]$ & $20\,[8,42]$ & $25\,[11,47]$ & $45\,[36,55]$ \\
    FQL  & $65\,[43,82]$ & $70\,[48,85]$ & $75\,[53,89]$ & $25\,[11,47]$ & $20\,[8,42]$ & $51\,[41,61]$ \\
    DSRL & $60\,[39,78]$ & $75\,[53,89]$ & $65\,[43,82]$ & $30\,[15,52]$ & $20\,[8,42]$ & $50\,[40,60]$ \\
    \textbf{MF-AWLS (ours)} & $\mathbf{75}\,[53,89]$ & $\mathbf{85}\,[64,95]$ & $\mathbf{80}\,[58,92]$ & $\mathbf{45}\,[26,66]$ & $\mathbf{40}\,[22,61]$ & $\mathbf{65}\,[55,74]$ \\
    \bottomrule
  \end{tabular}
\end{table}

\paragraph{Success rates with confidence intervals.}
Table~\ref{tab:realworld_ci} expands the main-paper success counts
(Table~\ref{tab:realworld}) with Wilson score $95\%$ confidence intervals computed
directly from the $20$ rollouts per task (and $n\!=\!100$ in aggregate); no
additional experiments are involved. At the per-task level the $20$-rollout
intervals are necessarily wide and overlap across methods, so individual-task
differences are not individually significant and we read the real-world result at
the aggregate level. There, MF-AWLS reaches $65/100$ ($[55.3,73.6]\%$) and its
interval is disjoint from the frozen MeanFlow base ($45/100$, $[35.6,54.8]\%$). The
marginal intervals of MF-AWLS and the strongest Q-maximization baselines (FQL
$51/100$, $[41.3,60.6]$; DSRL $50/100$, $[40.4,59.6]$) overlap. With five tasks of
$20$ rollouts each, the appropriate analysis stratifies by task and corrects for the
three comparisons rather than pooling into a single binomial. A Cochran--Mantel--Haenszel
test stratified by task gives MF-AWLS vs.\ MeanFlow $p\!=\!0.004$, vs.\ DSRL
$p\!=\!0.031$, and vs.\ FQL $p\!=\!0.042$; after Holm correction for the three
comparisons, only the improvement over the frozen behavior-cloning base (MeanFlow)
is significant at $\alpha\!=\!0.05$, while the gains over FQL and DSRL are nominally
below $0.05$ but \emph{do not survive} the correction. We therefore treat the
aggregate comparison as descriptive, rest the significance claim only on the gain
over the BC base, and report the FQL/DSRL margins as suggestive; the per-task
numbers are shown transparently, since $n\!=\!20$ per task is underpowered for
task-level claims.

\end{document}